\newif\ifJournal

\ifJournal

\else
\documentclass[11pt]{article} 
\RequirePackage{amsthm,amsmath,amsfonts,amssymb, mathtools}
\RequirePackage{graphicx}
\RequirePackage{natbib}

\usepackage{authblk} 
\usepackage[dvipsnames]{xcolor}
\usepackage[plainpages=false,pdfpagelabels,
colorlinks=true,linkcolor=RedOrange,
urlcolor=RedOrange,citecolor=MidnightBlue]{hyperref}
\providecommand{\keywords}[1]{\small \textbf{\textit{Keywords---}} #1}
\usepackage{fullpage}

\usepackage{my-macros}

\usepackage{subfig}
\usepackage{tikz}
\usepackage{float}
\usepackage{calrsfs}

\newcommand{\N}{\mathbb{N}}
\usepackage{enumitem}
\setlist[enumerate]{noitemsep, topsep=0pt}
\setlist[itemize]{noitemsep, topsep=0pt}

\title{Online Learning to Transport via the Minimal Selection Principle}
\begin{document}

\author[1]{Wenxuan Guo}
\author[1]{YoonHaeng Hur}
\author[1]{Tengyuan Liang}
\author[2]{Christopher Thomas Ryan}

\affil[1]{University of Chicago}
\affil[2]{University of British Columbia}


\maketitle

\begin{abstract}
Motivated by robust dynamic resource allocation in operations research, we study the \textit{Online Learning to Transport} (OLT) problem where the decision variable is a probability measure, an infinite-dimensional object. We draw connections between online learning, optimal transport, and partial differential equations through an insight called the minimal selection principle, originally studied in the Wasserstein gradient flow setting by \citet{Ambrosio_2005}. This allows us to extend the standard online learning framework to the infinite-dimensional setting seamlessly. Based on our framework, we derive a novel method called the \textit{minimal selection or exploration (MSoE) algorithm} to solve OLT problems using mean-field approximation and discretization techniques. In the displacement convex setting, the main theoretical message underpinning our approach is that minimizing transport cost over time (via the minimal selection principle) ensures optimal cumulative regret upper bounds. On the algorithmic side, our MSoE algorithm applies beyond the displacement convex setting, making the mathematical theory of optimal transport practically relevant to non-convex settings common in dynamic resource allocation.
\end{abstract}


\tableofcontents
\fi


\section{Introduction}
\label{sec:intro}

Online learning and online convex optimization offer an elegant framework for regret minimization under worst-case sequences \citep{gordon1999regret,zinkevich2003online,shalev-shwartz2012OnlineLearning, orabona2021ModernIntroduction}. Principles for these methods have been discovered independently in decision theory, game theory, learning theory, and convex optimization \citep{cesa2006prediction}. For problems with a finite-dimensional decision variable, extensive studies have been conducted to understand online algorithms that achieve small regret, either in Euclidean or non-Euclidean settings. Such algorithms usually rely on some notion of (sub-)gradients to iteratively improve the finite-dimensional decision variable. In the simplest Euclidean setting, let $\ell_t \colon \R^d \rightarrow \R$ be a smooth convex function and $x_t \in \R^d$ be a finite-dimensional decision variable, both indexed by time $t \in \mathbb{N}$. Online gradient descent with stepsize $\eta$ satisfies, for any $z \in \R^d$,
\begin{align}
	\label{eqn:euclidean}
	2\eta\big( \overbrace{\ell_t(x_t) - \ell_t(z)}^{\text{regret term}} \big) - \big( \overbrace{\|x_{t} - z \|^2 - \| x_{t+1} - z \|^2}^{\text{telescoping term}} \big) \leq  \overbrace{ \| x_{t} - x_{t+1} \|^2}^{\text{transport cost}} \leq  \overbrace{ \| \eta \nabla \ell_t (x_t) \|^2}^{\text{gradient in Euclidean norm}} .
\end{align}
In a simple non-Euclidean setting, let $p_t \in \Delta_d$ be a probability distribution on a $d$-discrete decision and $\ell_t \in \R^d$ be a sequence of cost vectors where each element $\ell_t[k]$ denotes the cost associated with the $k$-th action and $\ell_t(p_t) := \sum_{k = 1}^{d} \ell_t[k] p_t[k]$ thus $\nabla \ell_t (\cdot) \equiv \ell_t $. Online mirror descent (specifically, exponentiated gradient) satisfies for any $q \in \Delta_d$
\begin{align}
	\label{eqn:entropy}
	\eta \big( \overbrace{\ell_t(p_t) - \ell_t(q)}^{\text{regret term}} \big) - \big( \overbrace{ d_{\mathrm{KL}}(q || p_{t}) - d_{\mathrm{KL}}(q || p_{t+1})}^{\text{telescoping term}} \big) \leq   \overbrace{ d_{\mathrm{KL}} (p_t || p_{t+1} ) }^{\text{transport cost}} \leq  \overbrace{ \| \eta \nabla \ell_t \|_{p_t}^2 }^{\text{gradient in local norm}} ,
\end{align}
where the local norm is defined as $\|\ell_t\|_{p_t}^2 := \sum_{k = 1}^{d} (\ell_t[k])^2 p_t[k]$. It is immediate to spot the resemblance between \eqref{eqn:euclidean} and \eqref{eqn:entropy}. The telescoping term marks the progress towards the competing decision ($z$ or $q$), and the first right-hand side of the expression describes some notion of ``transport cost'' induced by different geometries.

It is natural to wonder how the theoretical and algorithmic principles behind these finite-dimensional cases extend to the infinite-dimensional case, specifically when the decision variable is a probability distribution over a generic metric space. Such a question is of practical importance. Several dynamic resource allocation problems in operations research involve these infinite-dimensional objects, e.g., routing a network of drones over airspace or allocating drivers across a city topology in ridesharing. Theoretically, the infinite-dimensional object can be viewed as a generalization of both \eqref{eqn:euclidean} ($x$ from finite to infinite-dimensional) and \eqref{eqn:entropy} ($p$ from discrete to continuous measure). As hinted in the previous paragraph where small ``transport cost'' terms guarantee small regret, optimal transport theory \citep{Villani_2003, Ambrosio_2005} plays a pivotal role in extending online learning to the infinite-dimensional case. This paper draws connections between online learning, optimal transport, and partial differential equations (PDEs) using an insight called the minimal selection principle. Eventually, we derive new algorithms based on this insight.

Optimal transport (OT) studies how to move in the space of probability measures to minimize a cost metric. This cost metric, credited to Monge in 1781, is associated with the optimal way to move mass between two probability measures. This paper investigates online learning using a toolbox set out by Breiner, who brought perspectives from PDEs, geometry, and functional analysis to the study of OT around 1987 \citep{Brenier_1987, Brenier_1989}. Let us first introduce our infinite-dimensional online optimization problem in the language of OT and then lay out the connection between OT and online learning. Consider the following \textbf{Online Learning to Transport (OLT)} problem: let $\cP(\R^d)$ be the space of probability measures on $\R^d$ where, in each round, $t = 1, \ldots, T$, 
\begin{itemize}
	\item an adversary chooses an energy functional $\cE_t \colon \cP(\R^d) \rightarrow \R$ without revealing it;
	\item the player commits a probability measure $\mu_t \in \cP(\R^d)$ as the decision variable;
	\item the adversary reveals the energy functional, and the player suffers the loss $\cE_t(\mu_t)$.
\end{itemize}
The player needs to decide the next $\mu_{t+1}$ based on $\mu_{t}$ and all historical information and aims to minimize cumulative regret from facing the adversary. The main conceptual message we discover in OLT is that optimally transporting the probability measures $\mu_{1} \rightarrow \mu_{2} \rightarrow \ldots \rightarrow \mu_{T}$ with respect to an appropriate cost enables small cumulative regret.

To provide a glimpse of the angle that inspired us to connect online learning and OT, we follow a viewpoint taken in \cite{benamou1999numerical} and \cite{otto2001geometry}. All concepts introduced in this paragraph have rigorous definitions in Section~\ref{sec:background}, so we proceed at a high level here. Mimicking the finite-dimensional case, OT theory provides a way to calculate a type of (sub-)gradient of $\cE_t$ over $\cP(\R^d)$ when equipped with a certain metric. This (sub-)gradient is used to update $\mu_{t} \rightarrow \mu_{t + 1}$. To enrich this analogy, we consider a continuous/infinitesimal time analog. In the finite-dimensional cases \eqref{eqn:euclidean} and \eqref{eqn:entropy}, the first right-hand side in a continuous time analog would quantify movement according to the ODE $\dv{x_t}{t} = - \nabla \ell_t (x_t)$ where
	$\big\| \dv{ x_t }{t} \big\|^2 = \| \nabla \ell_t (x_t) \|^2$.
In the infinite-dimensional case, the density $\rho_t$ (of $\mu_t$ w.r.t. the Lebesgue measure) evolves according to the PDE $\pdv{\rho_t}{t} = \nabla \cdot (\rho_t \boldsymbol{\xi}_t)$ with a vector field $\boldsymbol{\xi}_t \colon \R^d \to \R^d$ from the ``Fr\'{e}chet subdifferential'' $\partial \cE_t(\mu_t)$. The infinitesimal transport cost is\footnote{The curious reader may identify the above resemblance to the local norm in \eqref{eqn:entropy}: here the local Riemannian geometry quantifies the transport cost.} 
\begin{align}
	\big\| \pdv{ \rho_t}{t} \big\|_{\rho_t}^2 = \inf_{\boldsymbol{\xi} \in L^2( \mu_t; \R^d)} \left\{ \int  \| \boldsymbol{\xi}(x) \|^2 \dd{\mu_t(x)} ~:~  \boldsymbol{\xi} \in \partial \cE_t(\mu_t)  \right\} \;. \label{eqn:variation}
\end{align}
This motivates the \emph{minimal selection principle} of \cite{Ambrosio_2005}: among all the possible vector fields that belong to the Fr\'{e}chet subdifferential, we aim to select one whose kinetic energy (captured by the integral $\int \| \boldsymbol{\xi}(x) \|^2 \dd{\mu_t(x)}$) is lowest. Assuming a notion of convexity of $\cE_t$ along Riemannian geodesics, we formally show that minimizing the transport cost over time using the minimal selection principle yields the smallest cumulative regret upper bounds in the OLT problem.

Equipped with the minimal selection principle, we propose a novel algorithm in Section~\ref{sec:zeroth} that we call \emph{minimal selection or exploration (MSoE)} to solve the variational optimization problem in \eqref{eqn:variation} and, in turn, solve OLT using only zeroth-order partial feedback similar to the bandit setting. To make this algorithm numerically tractable, we use a discrete analog of the minimal selection principle using mean-field approximations. It is noteworthy that our MSoE algorithm works beyond the convex setting, which shows how elegant theory from OT is practically relevant to certain non-convex settings common in robust dynamic resource allocation. Simple toy examples demonstrating the algorithm's empirical performance are discussed in Appendix \ref{sec:simulations}.

A key feature of this paper is the natural simplicity of its arguments and algorithmic principles that become apparent once the framework connecting online learning and OT is established. Several directions for extending our framework are discussed at the end of the paper.

\paragraph{Related work} Our paper is at the intersection of two active research fields: online learning and optimal transport. For the former, due to space limits, we cannot do fair justice to credit all contributions properly; see \cite{orabona2021ModernIntroduction} for a comprehensive survey. Here, we give a selective overview. Several improvements of online gradient descent \eqref{eqn:euclidean} using an adaptive stepsize have been proposed \citep{streeter2010less, mcmahan2010adaptive} with a further modification via scale-freeness \citep{orabona2015,orabona2018}. Expression \eqref{eqn:entropy} can be derived using different principles such as exponentiated gradient \citep{KIVINEN19971}, online mirror descent (OMD), follow-the-regularized-leader (FTRL) \citep{shalev-shwartz_primal-dual_2007, Abernethy08competingin}. Modifications of FTRL via adaptive regularization have been introduced \citep{vanErven_2011,rooij_2014, orabona2015}. General first-order Riemannian optimization methods have been investigated in \cite{pmlr-v49-zhang16b}. Using martingale tools, a theoretical framework for sequential prediction has been studied in \cite{rakhlinStatisticalLearning}. The typical regret bound for online learning with $K$-discrete actions scales as $\sqrt{T \log K}$ in the full information setting, and as $\sqrt{T K \log K}$ in the bandit/partial feedback setting. Therefore, naive generalizations of these bounds to the infinite-dimensional case ($K \to \infty$) result in diverging bounds. It is thus unclear whether $\sqrt{T}$-regret is achievable in the infinite-dimensional case. We employ tools from optimal transport to resolve this issue and obtain optimal bounds.

Beyond well-established analytic results of OT, computational \citep{cuturi_2013, genevay_etal_2016, Altschuler_weed_rigollet_2017} and statistical \citep{weed_2019, liang2018HowWell, weed2019estimation,  liang2019EstimatingCertain, hutter_rigollet_2021} aspects have been emerging research areas at the intersection of OT, probability distribution estimation, and numerical sampling. At the same time, OT has been a versatile tool for various applied tasks such as image retrieval \citep{rubner_earth_2000}, computational linguistics \citep{kusner}, and domain adaptation \citep{courty_2017}. One of the successful applications of OT is generative sampling, such as GANs \citep{goodfellow_2014}. OT has improved generative sampling in several ways: by modifying GANs using OT-based probability metrics \citep{arjovsky_2017, genevay_2018} or by utilizing dual formulations of OT \citep{seguy_damodaran_flamary_courty_rolet_blondel_2018, makkuva_taghvaei_oh_lee_2020}. More recently, further improvements \citep{bunne_alvarez-melis_krause_jegelka_2019,hur_guo_liang_2021} have been made by considering the Gromov-Wasserstein \citep{memoli_2011}, a generalization of OT ideas for identifying isomorphism in metric measure spaces.

\paragraph{Notation} Let $\cP(\R^d)$ denote the set of all Borel probability measures on $\R^d$. Let $\cP^r(\R^d)$ denote the subset of $\cP(\R^d)$ of Borel probability measures that are absolutely continuous with respect to the Lebesgue measure $\cL^d$ on $\R^d$. For $\mu \in \cP^r(\R^d)$, we write $\mu = \rho \cdot \cL^d$ to signify that $\rho$ is a density function of $\mu$ with respect to $\cL^d$. Let $\Pi(\mu, \nu)$ denote the collection of all couplings of $\mu, \nu \in \cP(\R^d)$, that is, $\gamma \in \Pi(\mu, \nu)$ is a Borel probability measure on $\R^d \times \R^d$ such that $\gamma(A \times \R^d) = \mu(A)$ and $\gamma(\R^d \times B) = \nu(B)$ for all Borel subsets $A, B \subset \R^d$. For a measurable map $\bt \colon \R^d \to \R^p$ and $\mu \in \cP(\R^d)$, we define the pushforward measure $(\bt_{\#} \mu)(B) = \mu\{x \in \R^d : \bt(x) \in B\}$ for any Borel subset $B \subset \R^p$; hence $\bt_{\#} \mu \in \cP(\R^p)$. We let $\delta_x$ denote the Dirac measure for $x \in \R^d$ and $\bi \colon \R^d \to \R^d$ the identity map $\bi(x) = x$ for all $x \in \R^d$. Lastly, $\|\cdot\|$ denotes the standard Euclidean norm on $\R^d$, $[x]_+ = \max(x, 0)$ for $x \in \R$, and let $[n]$ denote the unordered set $\{1, \ldots, n\}$ for $n \in \N$.

\section{Optimal Transport, Minimal Selection Principle, and Regret Bound}
\label{sec:background}
In this section, we present our main theoretical results for OLT. As noted earlier, the key to our analysis is a connection with optimal transport theory. We start by introducing the notion of Wasserstein space to structure the decision space of OLT. Recall from \eqref{eqn:euclidean} and \eqref{eqn:entropy} that a notion of difference (transport cost) between two consecutive decisions plays a role in deriving a regret bound. To apply this idea to OLT, we utilize the Wasserstein distance between probability measures (decisions in OLT). Next, we discuss differential calculus over Wasserstein space, which serves as a key building block for deriving a regret bound for OLT. In Wasserstein space, we first define a subdifferential and then select an element associated with the smallest size, measured by the local Riemannian geometry of the Wasserstein space. Such an element, which we call a minimal selection, is defined by a variational problem and serves as a functional gradient. Based on this, we make transparent a strategy to obtain $\sqrt{T}$-regret that generalizes seamlessly to the infinite-dimensional setting. Lastly, we mention a connection between OLT and Wasserstein gradient flow. 

\subsection{Wasserstein Space}
One of the most important consequences of optimal transport theory is that we can define a distance between $\mu, \nu \in \cP(\R^d)$ by finding a coupling that gives the smallest transport cost. Concretely, we define the Wasserstein distance $W_2$ on $\cP(\R^d)$ as $W_2^2(\mu, \nu) = \inf_{\gamma \in \Pi(\mu, \nu)} \int_{\R^d \times \R^d} \|x - y\|^2 \dd{\gamma(x, y)}$. The Wasserstein distance $W_2$ is indeed a distance over $\cP_2(\R^d)$, where $\cP_2(\R^d) = \{\mu \in \cP(\R^d) : \int_{\R^d} \|x\|^2 \dd{\mu(x)} < \infty\}$. From this, we obtain a metric space $(\cP_2(\R^d), W_2)$ of probability measures called the Wasserstein space.

One important property of $W_2$ is that we can find a unique optimal coupling under certain regularity conditions. Here, an optimal coupling is any coupling $\gamma \in \Pi(\mu, \nu)$ associated with the smallest transport cost, that is, $W_2^2(\mu, \nu) = \int_{\R^d \times \R^d} \|x - y\|^2 \dd{\gamma(x, y)}$. Proposition 2.1 of \cite{Villani_2003} tells that $\Pi_o(\mu, \nu)$, the set of all optimal couplings, is always nonempty. The following result states that $\Pi_o(\mu, \nu)$ is a singleton if $\mu$ is absolutely continuous with respect to the Lebesgue measure; moreover, such an optimal coupling is concentrated on the graph of some unique map. 

\begin{lemma}
	\label{def:brenier-partial}
	Let $\cP_2^r(\R^d) = \cP_2(\R^d) \cap \cP^r(\R^d)$. Given $\mu \in \cP_2^r(\R^d)$ and $\nu \in \cP_2(\R^d)$, there exists a unique optimal coupling $\gamma \in \Pi(\mu, \nu)$, namely $\Pi_o(\mu, \nu) = \{\gamma\}$. Also, there exists a unique map $\bt_{\mu}^{\nu} \colon \R^d \to \R^d$ such that $\gamma$ is concentrated on the graph of $\bt_{\mu}^{\nu}$, that is, 
	\begin{equation*}
		\gamma (\{(x, y) \in \R^d \times \R^d : y = \bt_{\mu}^{\nu}(x)\}) = 1 \; .
	\end{equation*}
	This implies $(\bt_{\mu}^{\nu})_\# \mu = \nu$ and $W_2^2(\mu, \nu) = \int_{\R^d} \|x - \bt_{\mu}^{\nu}(x)\|^2 \dd{\mu(x)}$.
\end{lemma}
\begin{remark}
	This is a part of Brenier's theorem \citep{Brenier_1991}, which contains further properties of $\bt_{\mu}^{\nu}$ such as monotonicity. See Theorem 2.12 of \cite{Villani_2003} for details.
\end{remark}

\subsection{Displacement Convexity and Subdifferential Calculus}
Having defined the decision space for OLT, we explore how to minimize a loss function over this space. As noted earlier, the key is finding an object analogous to the gradient in the finite-dimensional setting. \cite{Ambrosio_2005} achieves this by rigorously defining notions of convexity and differential calculus on Wasserstein space. We reproduce a few relevant results from \cite{Ambrosio_2005}. Throughout, we consider a functional $\cE \colon \cP_2(\R^d) \to (-\infty, \infty]$ with a nonempty domain $D(\cE) := \{\mu \in \cP_2^r(\R^d) : \cE(\mu) < \infty\} \neq \emptyset$; for a concise summary, we restrict the domain to $\cP_2^r(\R^d)$, see Section 10.1 of \cite{Ambrosio_2005} for more details. First, we discuss convexity of $\cE$. To import convexity into $\cP^r_2(\R^d)$, we need a concept that replaces the concept of `line segment' in a vector space. \cite{McCann_1997} introduces the displacement interpolation for connecting two elements of $\cP_2(\R^d)$ and defines the convexity based on it.
\begin{definition}[Displacement interpolation and Convexity]
	\label{def:displacement-convex}
	We define the displacement interpolation between $\mu, \nu \in \cP_2^r(\R^d)$ as a curve $(\pi_\eta^{\mu \to \nu})_{\eta \in [0, 1]}$ in $\cP_2^r(\R^d)$ such that
	\begin{equation*}
		\pi_\eta^{\mu \to \nu} := ((1- \eta) \bi + \eta \bt_{\mu}^{\nu})_{\#} \mu \; .
	\end{equation*}
	We say $\cE$ is displacement convex if $\cE(\pi_\eta^{\mu \to \nu}) \leq (1 - \eta) \cE(\mu) + \eta \cE(\nu)$ for all $\mu, \nu \in D(\cE)$ and $\eta \in [0, 1]$.
\end{definition}

\begin{remark}
	To see that the displacement interpolation $(\pi_\eta^{\mu \to \nu})_{\eta \in [0, 1]}$ serves as a segment connecting $\mu$ and $\nu$ one can verify that $\pi_0^{\mu \to \nu} = \mu$, $\pi_1^{\mu \to \nu} = \nu$, and $\pi_\eta^{\mu \to \nu} \in \cP_2^r(\R^d)$ for all $\eta \in (0, 1)$. See Chapter 7 of \cite{Ambrosio_2005} for details.
\end{remark}

Next, we introduce the Fr\'{e}chet subdifferential, which delineates the differentiable structure on the Wasserstein space. Recall that derivatives or subgradients of functionals defined on a Hilbert space are linear functionals over that space. To mimic their features, we define the subdifferential at $\mu \in \cP_2^r(\R^d)$ by means of a Hilbert space $L^2(\mu; \R^d)$ as follows.

\begin{definition}[Fr\'{e}chet Subdifferential]
	\label{def:frechet-differential}
	For $\mu \in \cP_2(\R^d)$, let $L^2(\mu; \R^d)$ be the collection of vector fields $\boldsymbol{\xi} \colon \R^d \to \R^d$ satisfying $
	\|\boldsymbol{\xi}\|_{L^2(\mu; \R^d)}^2 := \int_{\R^d} \|\boldsymbol{\xi}(x)\|^2 \dd{\mu(x)} < \infty$. For a lower semi-continuous functional $\cE$ and $\mu \in D(\cE)$, we say that $\boldsymbol{\xi} \in L^2(\mu; \R^d)$ belongs to the Fr\'{e}chet subdifferential $\partial \cE(\mu)$ if 
	\begin{equation}
		\label{eq:frechet}
		\liminf_{\nu \to \mu} \frac{\cE(\nu) - \cE(\mu) - \int_{\R^d} \langle \boldsymbol{\xi}(x), \bt_{\mu}^{\nu}(x) - x \rangle \dd{\mu(x)}}{W_2(\nu, \mu)} \ge 0 \; ,
	\end{equation}
	where $\liminf_{\nu \to \mu}$ is based on the convergence under $W_2$. 
\end{definition}

Definition \ref{def:frechet-differential} is a local definition as one can see from \eqref{eq:frechet}. For a displacement convex functional, however, Fr\'{e}chet subdifferentials admit a global characterization. Moreover, as the next result shows, we can find a unique member of the Fr\'{e}chet subdifferential with the smallest norm; see Section 10.1 of \cite{Ambrosio_2005}.

\begin{lemma}[Minimal Selection Principle]
	\label{def:minimal-selection}
	Let $\cE$ be a lower semi-continuous and displacement convex functional, then a vector field $\boldsymbol{\xi}$ belongs to the Fr\'{e}chet subdifferential $\partial \cE(\mu)$ if and only if 
	\begin{equation*}
		\cE(\nu) - \cE(\mu) \geq \int_{\R^d} \langle \boldsymbol{\xi}(x), \bt_{\mu}^{\nu}(x) - x \rangle \dd{\mu(x)} \quad \forall \nu \in \cP_2^r(\R^d) \; .
	\end{equation*}
	In this case, $\partial \cE(\mu)$ has a unique element $\partial^{o} \cE(\mu)$ called the \emph{minimal selection} with the smallest norm in the following sense:
	\begin{equation*}
		\partial^{o} \cE(\mu) = \argmin \{ \|\boldsymbol{\xi} \|^2_{L^2(\mu;\R^d)} : \boldsymbol{\xi} \in \partial \cE(\mu)\} \; .
	\end{equation*}
\end{lemma}

As we shall see in Theorem~\ref{cor:OLT-discrete}, the minimal selection $\partial^{o} \cE(\mu)$ plays a role analogous to a gradient in providing a regret bound for OLT comparable to \eqref{eqn:euclidean} and \eqref{eqn:entropy}. We conclude this subsection with two canonical examples of the minimal selection principle.

\begin{example}[Potential Functional]
	\label{def:potential}
	Given a lower semi-continuous function $V \colon \R^d \rightarrow (-\infty, \infty]$, we call $\cV \colon \cP_2(\R^d) \rightarrow (-\infty, \infty]$ a potential functional associated with $V$ if
	\begin{equation*}
		\cV(\mu) := \int_{\R^d} V(x) \dd{\mu(x)}
	\end{equation*}
	and $\partial^o \cV(\mu) = \nabla V$ if $V$ is convex and satisfies some regularity conditions; see Section 10.4 of \cite{Ambrosio_2005}.
\end{example}

\begin{example}[Interaction Functional]
	\label{def:interaction}
	Given a lower semi-continuous function $W \colon \R^d \rightarrow [0, \infty)$, we call $\cW \colon \cP_2(\R^d) \rightarrow [0, \infty]$ an interaction functional associated with $W$ if
	\begin{equation*}
		\cW(\mu) := \int_{\R^d} \int_{\R^d} W(x - y) \dd{\mu(x)} \dd{\mu(y)} \; 
	\end{equation*}
	and $\partial^o \cW(\mu) = \nabla W * \rho$, where $\mu = \rho \cdot \cL^d$, if $W$ is convex and satisfies some regularity conditions; see Section 10.4 of \cite{Ambrosio_2005}. Here, $*$ denotes a convolution of vector field $\nabla W$ and density $\rho$. 
\end{example}

\subsection{Regret Bound for OLT}
\label{sec:regret-bound-OLT}
We are ready to derive a regret bound for the OLT problem. The following result, often referred to as Evolution Variational Inequality (EVI) in the PDE literature, clarifies why the minimal selection principle is key to obtaining a regret bound. 

\begin{lemma}[EVI]
	\label{prop:EVI}
	Let $\cE$ be a lower semi-continuous and displacement convex functional. Let $\mu \in D(\cE)$ and $\boldsymbol{\xi} \in \partial \cE(\mu)$. Fix $\eta > 0$ and define $\mu_\eta := (\bi - \eta \boldsymbol{\xi})_{\#} \mu$. Then, for any $\nu \in \cP_2^r(\R^d)$,
	\begin{equation}
		\label{eq:EVI-first}
		\cE(\mu) - \cE(\nu) \leq \frac{W_2^2(\mu, \nu) - W_2^2(\mu_\eta, \nu)}{2 \eta} + \frac{\eta}{2} \int_{\R^d} \|\boldsymbol{\xi}(x)\|^2 \dd{\mu(x)} \; .
	\end{equation}
\end{lemma}
\begin{proof}
	Recall from Lemma \ref{def:brenier-partial} that $W_2^2(\mu, \nu) = \int_{\R^d} \| x - \bt_{\mu}^\nu (x) \|^2 \dd{\mu(x)}$.	Let $(\bi - \eta\boldsymbol{\xi}, \bt_{\mu}^{\nu})$ be the map from $\R^d$ to $\R^{2 d}$ that maps $x \mapsto (x - \eta \boldsymbol{\xi}(x), \bt_{\mu}^{\nu}(x))$, then $(\bi - \eta \boldsymbol{\xi}, \bt_{\mu}^\nu )_{\#} \mu \in \Pi(\mu_\eta, \nu)$	and thus
	\begin{equation*}
		W_2^2(\mu_\eta, \nu) = \inf_{\gamma \in \Pi(\mu_\eta, \nu)} \int_{\R^d} \|x - y\|^2 \dd{\gamma(x, y)} \leq \int_{\R^d} \| (x - \eta \boldsymbol{\xi}(x)) - \bt_{\mu}^\nu (x) \|^2 \dd{\mu(x)} \; .
	\end{equation*}
	Hence,
	\begin{align*}
		\frac{W_2^2(\mu, \nu) - W_2^2(\mu_\eta, \nu)}{2\eta} 
		& \geq \frac{1}{2\eta} \int_{\R^d} \left(\| x - \bt_{\mu}^\nu (x) \|^2 - \| x - \eta \boldsymbol{\xi}(x) - \bt_{\mu}^\nu (x) \|^2 \right) \dd{\mu(x)} \\
		& = \int_{\R^d} \langle \boldsymbol{\xi}(x), x - \bt_{\mu}^\nu (x) \rangle \dd{\mu(x)} - \frac{\eta}{2} \int_{\R^d} \|\boldsymbol{\xi}(x) \|^2 \dd{\mu(x)} \\
		& \ge \cE(\mu) - \cE(\nu) - \frac{\eta}{2} \int_{\R^d} \|\boldsymbol{\xi}(x) \|^2 \dd{\mu(x)} \quad (\because \text{Lemma \ref{def:minimal-selection}}) \; .
	\end{align*}	
\end{proof}
From \eqref{eq:EVI-first}, it is now obvious that the minimal selection $\boldsymbol{\xi} = \partial^{o} \cE(\mu)$ gives the smallest upper bound. Also, notice the similarity of \eqref{eqn:euclidean}, \eqref{eqn:entropy}, and \eqref{eq:EVI-first}; the last term on the right-hand side of \eqref{eq:EVI-first} corresponds to the norm of a gradient in \eqref{eqn:euclidean} and \eqref{eqn:entropy}. The minimal selection enables us to import bounding techniques in \eqref{eqn:euclidean} and \eqref{eqn:entropy} into the OLT problem. Based on this connection, we propose a strategy to tackle the OLT: for each time $t$, the player finds the minimal selection $\boldsymbol{\xi}_t = \partial^{o} \cE_t(\mu_t)$ and updates $\mu_{t+1} = (\bi - \eta \boldsymbol{\xi}_t)_{\#} \mu_t$. Using \eqref{eq:EVI-first}, we obtain the following regret bound.

\begin{theorem}[Regret Bound: Discrete-Time] 
	\label{cor:OLT-discrete}
	Assume $\cE_t$ of the OLT is displacement convex for all $t \in [T]$ and $T \in \N$. The player selects $\boldsymbol{\xi}_t = \partial^{o} \cE_t(\mu_t)$ and updates $\mu_{t+1} = (\bi - \eta \boldsymbol{\xi}_t)_{\#} \mu_t$ for each round $t$, where $\eta > 0$ is a fixed stepsize. Then, for any $\nu \in \cP^r_2(\R^d)$, 
	\begin{equation}
		\label{eq:olt-discrete}
		\sum_{t = 1}^{T} \cE_t(\mu_t) -  \sum_{t = 1}^T \cE_t(\nu) \leq \frac{W_2^2(\mu_1, \nu) - W_2^2(\mu_{T+1}, \nu)}{2 \eta} + \frac{\eta}{2} \sum_{t = 1}^{T} \int_{\R^d} \|\boldsymbol{\xi}_t(x)\|^2 \dd{\mu_t(x)} \; .
	\end{equation}
\end{theorem}
\begin{remark}
	\label{rmk:regret}
 Under mild conditions such as $\mu_t$ and $\nu$ are supported on a bounded domain and $\boldsymbol{\xi}_t$ are uniformly bounded, we have $W_2(\mu_t, \nu) \le D$ and $\|\boldsymbol{\xi}_t\|_{L^2(\mu_t; \R^d)} \le L$ for all $t$ for some constants $D, L > 0$. Then, the right-hand side of \eqref{eq:olt-discrete} is upper bounded by $\frac{D^2}{2 \eta} + \frac{\eta L^2 T}{2}$, which becomes $D L \sqrt{T}$ for $\eta = \frac{D}{L \sqrt{T}}$ and yields $\sqrt{T}$-regret.
\end{remark}

\begin{remark}
	We clarify that the EVI is an existing technique in the PDE literature and has already been used for different purposes, for instance, see \cite{Durmus_2017} and \cite{Salim_2020}. We emphasize, however, that Theorem \ref{cor:OLT-discrete} utilizes such a technique for the first time in the context of the online learning given a sequence of adversarial functionals $(\cE_t)_{t = 1}^{T}$, and an arbitrary reference measure $\nu$.
\end{remark}

\subsection{Connections to Wasserstein Gradient Flow}

We conclude this section with a connection between online learning and Wasserstein gradient flow through a lens of continuous-time OLT, further highlighting why it is natural to employ insights from PDEs. In the infinitesimal limit as $\eta \rightarrow 0$, the algorithm solving OLT amounts to a Wasserstein gradient flow. To see this, consider the steepest descent on Wasserstein space according to the energy functional $\cE_t$ at time $t$, 
$
	\mu_{t+\eta} := \argmin_{\mu \in \cP_2(\R^d)} \cE_t(\mu) + \frac{1}{2\eta} W_2^2(\mu, \mu_t).
$
In the infinitesimal limit, the steepest descent defines continuous-time evolution of probability measures $(\mu_t)_{t \in [0, T]}$, naturally described by PDEs. The following continuity equation is the natural counterpart of discrete update of the steepest descent: letting $\mu_t = \rho_t \cdot \cL^d$, solve
\begin{equation}
	\label{eq:online-wgf}
	\frac{\partial \rho_t}{\partial t}
	= \nabla \cdot \left(\rho_t \boldsymbol{\xi}_{t}\right) \quad \text{where} ~~ \boldsymbol{\xi}_t \in \partial \cE_t(\mu_t) \; .
\end{equation}
Therefore, it is clear that our online OLT is a discrete analogue of the online Wasserstein gradient flow, with a time-varying energy functional $\cE_t$. A solution $(\mu_t)_{t \in [0, T]}$ enjoys the following regret bound; see Appendix \ref{sec:continuous-time} for the proof. 
\begin{theorem}[Regret Bound: Continuous-Time]
	\label{thm:continuous-regret}
	Assume $\cE_t$ is displacement convex for all $t \in [0, T]$. Under mild regularity assumptions, a solution $(\mu_t)_{t \in [0, T]}$ to \eqref{eq:online-wgf} satisfies 
	\begin{equation}
		\int_{0}^{T} \cE_t(\mu_t) \dd{t} - \int_{0}^{T} \cE_t(\nu) \dd{t}
		\le \frac{W_{2}^{2}(\mu_{0}, \nu) - W_{2}^{2}(\mu_{T}, \nu)}{2} \quad \forall \nu \in \cP^r_2(\R^d) \; . 
	\end{equation}
\end{theorem}
\begin{remark}
	If $\cE_t$ is time-invariant, say $\cE_t = \cE$ for all $t \in [0, T]$, then \eqref{eq:online-wgf} amounts to finding a solution to the standard Wasserstein gradient flow given $\cE$ (Chapter 11 of \cite{Ambrosio_2005}); minimizing the regret is exactly minimizing a functional $\cE$ by solving a Wasserstein gradient flow. Note that the RHS is time-independent, namely, we have bounded total regret and fast rate as $1/T$.
\end{remark} 

\section{Minimal Selection Algorithm with Zeroth-Order Information}
\label{sec:zeroth}

In this section, we propose a more practical framework for OLT and methods to solve it based on only zeroth-order, partial feedback. First, let us briefly explain our motivation for studying this practical framework. Recall that the strategy we studied in the previous section to solve OLT was to find the minimal selection $\boldsymbol{\xi}_t = \partial^{o} \cE_t(\mu_t)$ and update $\mu_t$ to $(\bi - \eta \boldsymbol{\xi}_t)_\# \mu_t$ for all $t$. The EVI gave us a regret bound as in Theorem \ref{cor:OLT-discrete}. However, such a strategy is difficult to implement in practice. Finding $\partial^{o} \cE_t(\mu_t)$ requires the player to access the Fr{\'e}chet subdifferential $\partial \cE_t(\mu_t)$ for each $t$. Such information is not directly available as nature typically only reveals the loss $\cE_t(\mu_t)$ and not the entire collection of vector fields in $\partial^{o} \cE_t(\mu_t)$. For instance, suppose $\cE_t$ is a potential functional associated with $V_t \colon \R^d \to \R$, then $\nabla V_t = \partial^{o} \cE_t(\mu_t)$ due to Example \ref{def:potential}. It is not obvious how the player can determine the vector field $\nabla V_t$ based only on the zeroth-order information $V_t$.

As a result, we need a more reasonable setting to implement a strategy based on the minimal selection principle. We achieve this goal by querying only through zeroth-order information, drawing a parallel to the bandit/partial feedback setting with finite arms. To make our discussion concrete, we focus on the case where $\cE_t = \cV_t$ (potential functional) that occurs frequently in applications; the loss functional is a potential functional associated with some $V_t \colon \R^d \to \R$ for all $t \in \N$. An extension to the interaction functional, another common setting in applications, will be discussed later in this section.  Under this framework, nature reveals $\cV_t(\mu_t)$ by telling the values of $V_t$ \textbf{only} on the support of $\mu_t$ and some fixed set of grid points. As we will see shortly, such a setting enables the player to execute a strategy based on the minimal selection principle, thereby obtaining a regret bound based on the EVI. In other words, instead of observing the gradient $\nabla V_t$ (minimal selection), the player approximates it using only zeroth-order information. As such, our formulation bridges the gap between the sophisticated theory in the previous section and its practical realization.

\subsection{OLT with Zeroth-Order Information and Minimal Selection Algorithm}
We formally define our online learning problem. Fix a domain $\Omega \subset \R^d$ and a set $Z \subset \Omega$ of grid points or hubs, say $Z = \{z_1, \ldots, z_n\}$, known to the player.\footnote{For instance, $z_i$'s could be hub points (in ridesharing applications), be stochastically sampled based on some reference measure, or be grid points for discretization. } At round $t \in \N$,
\begin{itemize}
	\item the player chooses a discrete measure $\mu_t := \frac{1}{m} \sum_{j = 1}^m \delta_{x_j^t}$, where we call $x_1^t, \ldots, x_m^t \in \Omega$ decision points,
	\item nature reveals a lower semi-continuous function $V_t \colon \R^d \to \R$ only on $\{x_j^t\}_{j \in [m]} \cup Z$, namely the zeroth-order information on the player's decision points and the fixed grid points, and the player suffers a loss given as the potential functional associated with $V_t$, that is, $\cV_t(\mu_t) = \frac{1}{m} \sum_{j = 1}^m V_t(x_j^t)$ (Example \ref{def:potential}).
\end{itemize}

Here, we view the grid points as the canonical locations of the domain $\Omega$. Meanwhile, decision points $x_j^t$ are any elements of $\Omega$, not necessarily the grid points. By using the zeroth-information $V_t(z_i)$, the player competes against the grid points, meaning that she aims to choose her decision points that would incur smaller losses than the grid points would do. Accordingly, we consider a regret $\sum_{t = 1}^T \cV_t(\mu_t) - \sum_{t = 1}^T \cV_t(\nu)$ for any $\nu \in \cP_2(\R^d)$ supported on $Z$.

Inspired by the minimal selection principle (Lemma \ref{def:minimal-selection}), we propose an algorithm for this learning problem.

\begin{definition}[Minimal Selection Algorithm]
	After round $t \in \N$, the player solves
	\begin{align}
		\min_{\xi_1, \ldots, \xi_m \in \R^d} \quad & \frac{1}{m} \sum_{j=1}^m \| \xi_j \|^2 \label{eq:obj1} \\
		\text{subject to} \quad & V_t(x_j^t) - V_t(z_i) \leq \langle \xi_j, x_j^t - z_i \rangle ~~ \forall (i, j) \in [n] \times [m] \; . \label{eq:obj2}
	\end{align}
	After obtaining a minimizer $(\xi_1^t, \ldots, \xi_m^t)$, the player updates $x_j^{t + 1} = x_j^{t} - \eta \xi_j^t$, where $\eta > 0$ is a suitable stepsize.
\end{definition}

\begin{remark}
	\label{rmk:unique}
	Note that this is a convex program. Also, there exists a unique minimizer $(\xi_1^t, \ldots, \xi_m^t)$ provided the constraint \eqref{eq:obj2} is feasible.
\end{remark}

Not surprisingly, this algorithm amounts to the minimal section principle in a discrete setting. Recall that $\|\boldsymbol{\xi}\|_{L^2(\mu_t; \R^d)}^2 = \frac{1}{m} \sum_{j = 1}^{m} \|\boldsymbol{\xi}(x_j^t)\|^2$ for any $\boldsymbol{\xi} \in L^2(\mu_t; \R^d)$, hence the objective function \eqref{eq:obj1} corresponds to the norm of an element of the Fr{\'e}chet subdifferential $\partial \cV_t(\mu_t)$. Meanwhile,  constraint \eqref{eq:obj2} amounts to finding a certified subgradient of $V_t$ at $x_j^t$. Recall from Example \ref{def:potential} that $\partial^{o} \cV_t(\mu_t) = \nabla V_t$. Essentially, this algorithm aims to find a minimizer $\xi_j^t = \nabla V_t(x_j^t)$. Also, the update rule $x_j^{t + 1} = x_j^t - \eta \xi_j^t$ corresponds to $\mu_{t + 1} = (\bi - \eta \boldsymbol{\xi}_t)_\# \mu_t$ in Theorem \ref{cor:OLT-discrete}, where $\boldsymbol{\xi}_t(x_j^t) = \xi_j^t$. 

Now, we can utilize the EVI to obtain a regret bound. By adapting the proof of Lemma \ref{prop:EVI}, we obtain the following result similar to Theorem \ref{cor:OLT-discrete}; see Appendix \ref{sec:prop:cvx} for the proof.

\begin{theorem}
	\label{prop:cvx}
	Assume $\Omega = \R^d$. Let $\mu_t$ and $\xi_j^t$ be the output of the minimal selection algorithm. If $V_t$ is convex for all $t \in \N$, then for any measure $\nu \in \cP_2(\R^d)$ supported on $Z$ and any $\eta>0$,
	\begin{equation}
		\label{eq:convex-upper-tele}
		\sum_{t = 1}^{T} \cV_t(\mu_t) - \sum_{t = 1}^{T} \cV_t(\nu) 
		\le 
		\frac{W_2^2(\mu_1, \nu) - W_2^2(\mu_{T + 1}, \nu)}{2 \eta} + \frac{\eta}{2} \sum_{t = 1}^{T} \left(\frac{1}{m} \sum_{j = 1}^m \|\xi_j^t\|^2\right) \; .
	\end{equation}
\end{theorem}
\begin{remark}
	Remark that the RHS of the above equation is precisely the minimal objective value in \eqref{eq:obj1}.
	Additionally, if $\Omega \subsetneq \R^d$, the update rule $x_j^{t + 1} = x_j^{t} - \eta \xi_j^t$ may cause $x_j^{t + 1} \notin \Omega$. We can prevent this via a projection step. See Subsection \ref{sec:extensions} for details.
\end{remark}

\begin{remark}	
	We emphasize that the key of the aforementioned zeroth-order framework is that the decision points can take any values in $\Omega$, allowing the support of $\mu_t$ to change continuously on $\Omega$. On the contrary, if we had restricted the support of $\mu_t$ to be contained in $Z$, the decision points would have moved discontinuously over $Z$. In applications such as real-time dynamic resource allocation, such a discontinuous movement is impractical as one needs to allocate the decision points (resources such as drones or drivers) within a short period of time. In conclusion, our framework is more suitable in practice due to the continuously varying support of $\mu_t$. At the same time, our framework preserves the infinite-dimensional nature.
\end{remark}

\subsection{Beyond Convex Case: Exploration and Regret Bound}
If $V_t$ is non-convex, constraint set \eqref{eq:obj2} might be infeasible. We propose a simple modification of the algorithm using exploration and provide a regret bound that extends beyond the convex case. 

Note that we may consider the previous learning problem at each decision point separately. The minimal selection algorithm amounts to solving, for each $j \in [m]$,
\begin{equation}
	\label{eq:feasible}
	\begin{aligned}
		\min_{\xi_j \in \R^d} \quad \|\xi_j\|^2 \quad \text{subject to} ~~ V_t(x_j^t) - V_t(z_i) \leq \langle \xi_j, x_j^t - z_i \rangle ~~ \forall i \in [n] \; .
	\end{aligned}
\end{equation}
If the above problem is not feasible, we sample a stochastic,  isotropic Gaussian vector with a suitable scaling and update $x_j^t$ along this exploration direction. 

\begin{definition}[Minimal Selection or Exploration (MSoE) Algorithm]
	\label{def:MSoE}
	After round $t \in \N$, let $S^t$ be a subset of $[m]$ such that $j \in S^t$ if there is $\xi_j \in \R^d$ satisfying $V_t(x_j) - V_t(z_i) \le \langle \xi_j, x_j^t - z_i \rangle$ for all $i \in [n]$. For $j \in S^t$, the player solves \eqref{eq:feasible} and updates $x_j^{t + 1} = x_j^t - \eta \xi_j^t$ using the unique solution $\xi_j^t$ to \eqref{eq:feasible}. For $j \notin S^t$, the player samples a Gaussian vector $g_j^t \sim N(0, I_d)$ and updates
	\begin{equation*}
		x_j^{t+1} = x_j^t - \sqrt{\eta \frac{\max_{i\in [n]} [V_t(x_j^t) - V_t(z_i)]_+}{d}}  g_j^t \; .
	\end{equation*}
	We assume all the Gaussian vectors are independent of each other. 
\end{definition}

In short, we modify the minimal selection algorithm to let the feasible decision points move along minimal selection directions as before, while infeasible decision points fully explore in a random direction. The fact that the discretization stepsize scales with $\sqrt{\eta}$ is motivated by Euler discretization of Langevin dynamics. The adaptive stepsize factor $\max_{i\in [n]} [V_t(x_j^t) - V_t(z_i)]_+$ follows the relative potential difference between infeasible decision points and grid points. For the MSoE algorithm, an expected regret bound can be derived that depends on the fraction of infeasible points; see Appendix \ref{sec:prop:rand} for the proof.

\begin{theorem}\label{prop:rand}
	Assume $\Omega = \R^d$. Let $\mu_t$, $\xi_j^t$, and $g_j^t$ be the output of the MSoE algorithm in Definition~\ref{def:MSoE}. For any $V_t$ and any measure $\nu \in \cP_2(\R^d)$ supported on $Z$,
	\begin{equation}
		\label{eq:bound-rand-total}
		\begin{aligned}
			\E\left[\sum_{t=1}^T \cV_t(\mu_t) - \sum_{t=1}^T \cV_t(\nu) \right] 
			& \leq \frac{W^2_2(\mu_1, \nu)}{2 \eta} + \frac{\eta}{2} \sum_{t = 1}^T \E\left[\frac{1}{m} \sum_{j \in S^t} \|\xi_j^t\|^2\right] \\
			 & \quad + \frac{3}{2} \sum_{t = 1}^{T} \E\left[\frac{1}{m} \sum_{j \notin S^t} \max_{i\in [n]} [V_t(x_j^t) - V_t(z_i)]_+\right] \; .
		\end{aligned}
	\end{equation}
\end{theorem}

From \eqref{eq:bound-rand-total}, we can expect a moderate regret bound provided that the proportion $\frac{|[m] \backslash S^t|}{m}$ of infeasible decision points is small. We prove that this is indeed the case under some assumptions on the sequence of functions $V_t$. To see this, we first define a feasible region.

\begin{definition} 
	For $t \in \N$, let	
	\begin{equation*}
		R^t := \left\{x \in \Omega : \exists \xi ~~ \text{such that} ~~ V_t(x) - V_t(z_i) \le \langle \xi, x - z_i \rangle ~~ \forall i \in [n]\right\} \; 
	\end{equation*}
	and for each $x \in R^t$ define 
	\begin{equation*}
		\xi_x = \argmin_{\xi \in \R^d} \|\xi\|^2 \quad \text{subject to} ~~ V_t(x) - V_t(z_i) \le \langle \xi, x - z_i \rangle ~~ \forall i \in [n] \; .
	\end{equation*}
	Lastly, for $\eta > 0$ let $R^t_{\eta} := \{x - \eta \xi_x : x \in R^t\}$.
\end{definition}

The set $R^t$ contains the points for which the minimal selection algorithm is feasible, hence $j \in S^t$ if and only if $x_j^t \in R^t$. Also, $\xi_x$ is well-defined for $x \in R^t$ as discussed in Remark \ref{rmk:unique}. Lastly, $R^t_{\eta}$ denotes the region obtained by updating points in $R^t$ according to the minimal selection algorithm. Now, we introduce an assumption to control the proportion of infeasible decision points. 

\begin{assumption}[Non-expansion condition]
	\label{asmp:non-expansion}
	There exists $\eta > 0$ so that $R^t_\eta \subseteq R^{t+1}$ for all $t \in \N$.
\end{assumption}

In other words, this assumption guarantees that any feasible point at some round is still feasible after the update; $j \in S^t$ implies $j \in S^{t + 1}$. Hence, the proportion $\frac{|[m] \backslash S^t|}{m}$ of infeasible points does not grow. We impose a further assumption that guarantees this proportion decreases.
\begin{assumption}[Shrinking condition]
	\label{asmp:shrinking}
	There exist $\eta > 0$ and $\gamma \in (0, 1)$ such that for all $t \in \N$,
	\begin{align}
		\inf_{x \in \Omega \backslash R^t} \Pr_{g \sim N(0, I_d)}\left(x - \sqrt{\eta \tfrac{\max_{i \in [n]} [V_t(x) - V_t(z_i)]_+}{d}} g \in R^{t+1} \right) \geq \gamma \; .
	\end{align}
\end{assumption}

In short, each infeasible point can be updated to a feasible point by a random direction with probability at least $\gamma$.\footnote{Assumptions \ref{asmp:non-expansion} and \ref{asmp:shrinking} are about relations between two adjacent functions $V_t$ and $V_{t + 1}$. We present a simple example of $V_t$ satisfying these assumptions in Appendix \ref{sec:assumptions}.} Combining this assumption with Assumption \ref{asmp:non-expansion}, we can show that the proportion $\frac{|[m] \backslash S^t|}{m}$ of infeasible decision points shrinks by a factor $1 - \gamma$ in each iteration. This observation yields the following regret bound. For simplicity, we assume each $V_t$ is uniformly bounded; see Appendix \ref{sec:prop:shrinking} for the proof.

\begin{theorem}[Shrinking fraction of infeasible decision points]
	\label{prop:shrinking}
	Assume $\Omega = \R^d$. Let $\mu_t$, $\xi_j^t$, and $g_j^t$ be the output of the MSoE algorithm. If Assumptions \ref{asmp:non-expansion} and \ref{asmp:shrinking} are satisfied with parameters $\eta, \gamma>0$, then for any measure $\nu \in \cP_2(\R^d)$ supported on $Z$,
	\begin{equation}
		\label{eq:shrinking-bound}
		\E\left[\sum_{t=1}^T \cV_t(\mu_t) - \sum_{t=1}^T \cV_t(\nu) \right] \leq \frac{W^2_2(\mu_1, \nu)}{2 \eta} + \frac{\eta}{2} \sum_{t = 1}^T \E\left[\frac{1}{m} \sum_{j \in S^t} \|\xi_j^t\|^2\right] + 3 B \gamma^{-1} \frac{|[m]\backslash S^1|}{m} \; ,
	\end{equation}
	where $B = \max_{t \in [T]} \sup_{x \in \Omega} |V_t(x)|$.
\end{theorem}

\subsection{Further Extensions}
\label{sec:extensions}

We briefly discuss a few extensions of our learning problem and the minimal selection algorithm.

\paragraph{Relaxed minimal selection algorithm}	We introduce another modification of the minimal selection algorithm using slack variables. After round $t \in \N$, for each $j \in [m]$, the player solves 
\begin{align}
	\min_{\xi_j \in \R^d, s_j \ge 0} \quad & \|\xi_j\|^2 + \frac{2}{\eta} s_j \\
	\text{subject to} \quad & -s_j + V_t(x_j^t) - V_t(z_i) \leq \langle \xi_j, x_j^t - z_i \rangle ~~ \forall i \in [n] \; . \label{eq:constraint-slack}
\end{align}
After finding minimizers $(\xi_j^t, s_j^t)$, the player updates $x_j^{t + 1} = x_j^{t} - \eta \xi_j^t$, where $\eta > 0$ is a suitable stepsize. Now that variables are $\xi_j$ and $s_j$, the constraint \eqref{eq:constraint-slack} is always feasible. The resulting optimization problem is still convex and admits a unique solution $(\xi_j^t, s_j^t)$. Also, for $j \in S^t$, that is, for a decision point for which the minimal selection algorithm is feasible, one can easily verify that the relaxed minimal selection algorithm boils down to the minimal selection algorithm. Moreover, we can obtain regret bounds similar to \eqref{eq:convex-upper-tele} or \eqref{eq:bound-rand-total}; see \eqref{eq:bound-slack} and \eqref{eq:slack-bound-rand} in Appendix \ref{sec:appendix-extensions}.

\paragraph{Interaction functional} We consider an extension of our problem to a different class of loss functional. Suppose we replace the potential functional with the interaction functional. Nature reveals a lower semi-continuous function $W^t \colon \R^d \to [0, \infty)$ only on $\{x_j^t\}_{j \in [m]} \cup Z$, that is, the player knows $W^t(x_j^t - z_i)$ for all $(i, j) \in [n] \times [m]$. The player suffers a loss given as the interaction functional associated with $W^t$ (Example \ref{def:interaction}):
\begin{equation}
	\label{eq:interaction-loss}
	\cW_t(\mu_t) 
	= \int_{\R^d} \int_{\R^d} W^t(x - y) \dd{\mu_t(x)} \dd{\mu_t(y)} 
	= \frac{1}{m^2} \sum_{j, k = 1}^m W^t(x_j^t - x_{k}^t) \; .
\end{equation}
In this case, we define the minimal selection algorithm as follows: after finishing round $t \in \N$, the player solves
\begin{align}
	\min_{\xi_1, \ldots, \xi_m \in \R^d} \quad & \frac{1}{m} \sum_{j = 1}^m \|\xi_j\|^2 \\
	\text{subject to} \quad & \frac{1}{m} \sum_{k = 1}^{m} W^t(x_j^t - x_k^t) - \min_{k \in [n]} W^t(z_i - z_k)  \leq \langle \xi_j, x_j^t - z_i \rangle ~~ \forall (i, j) \in [n] \times [m] \; . \label{eq:constraint-interaction}
\end{align}
Again, we have a convex program that admits a unique solution $(\xi_1^t, \ldots, \xi_m^t)$ provided the constraint \eqref{eq:constraint-interaction} is feasible. Using the EVI, we can derive a regret bound similar to \eqref{eq:convex-upper-tele}; see Theorem \ref{prop:interaction} in Appendix \ref{sec:appendix-extensions}.

\paragraph{Projection} We briefly mention a projection step to ensure $x_j^t \in \Omega$. The idea is to change the update rule from $x_j^{t + 1} = x_j^t - \eta \xi_j^t$ to $x_j^{t + 1} = P_{\Omega}(x_j^t - \eta \xi_j^t)$, where $P_{\Omega} \colon \R^d \to \Omega$ is defined as
\begin{equation*}
	P_{\Omega}(x) = \argmin_{\omega \in \Omega} \|x - \omega\| \; .
\end{equation*}
Recall that the projection $P_{\Omega}$ is well-defined provided $\Omega$ is closed and convex. Using this projection step, we can always ensure decision points are contained in $\Omega$. Moreover, this modification does not affect the regret bound \eqref{eq:convex-upper-tele}; see Theorem \ref{prop:projection} in Appendix \ref{sec:appendix-extensions}.

\section{Discussion}\label{sec:discussion}
In this paper, we have introduced and studied an infinite-dimensional online learning problem called the Online Learning to Transport (OLT) problem, where the decision variable is a probability measure on a fully nonparametric space, the Wasserstein space. Leveraging tools from optimal transport theory, we have established several theoretical results regarding the OLT problem; equipping the decision space with the Wasserstein distance, we have utilized the evolution variational inequality and the minimal selection principle to derive $\sqrt{T}$-regret bound for the OLT problem. Also, we have studied a practical framework for the OLT problem using zeroth-order information, where we develop a concrete algorithm called the Minimal Selection or Exploration (MSoE) that is numerically tractable and works beyond the convex setting.

Lastly, we mention a few directions for future research. First, theoretical results in Section \ref{sec:regret-bound-OLT} may be improved with a stronger notion of convexity. For instance, one might study if $\log(T)$-regret is achievable based on $\lambda$-convexity (Definition 9.1.1 of \cite{Ambrosio_2005}) of $\cE_t$. Another important direction is to modify the minimal selection strategy with adaptive stepsize $\eta_t$. In our work, $\sqrt{T}$-regret is achievable for a fixed stepsize depending on some universal constants as mentioned Remark \ref{rmk:regret}. Hence, to remedy this limitation, we might need to design an adaptive stepsize $\eta_t$ possibly in terms of $\boldsymbol{\xi}_{1}, \ldots, \boldsymbol{\xi}_{t - 1}$ (or their norms). Meanwhile, on the practical side, it would be interesting to see how the zeroth-order information framework and the MSoE algorithm work in real-life examples; for instance, we may compare our method with existing methods for real driver allocation datasets, which will shed light on developing even more practical framework and algorithms.

\section*{Acknowledgement}

Liang acknowledges the generous support from the NSF Career Award (DMS-2042473), and the William S. Fishman Faculty Research Fund at the University of Chicago Booth School of Business.

\bibliography{ref.bib}
\bibliographystyle{plainnat}
\nocite{*}

\appendix

\section{Proofs of the Regret Bounds in Section \ref{sec:zeroth}}
\subsection{Proof of Theorem \ref{prop:cvx}}
\label{sec:prop:cvx}
By convexity of $V_t$, the constraint set \eqref{eq:obj2} is always feasible, hence $\xi_j^t$ is well-defined. Let $\nu = \sum_{i} a_i \delta_{z_i}$ for some $a_1, \ldots, a_n \ge 0$ such that $\sum_{i = 1}^{n} a_i = 1$. We can find an optimal coupling between $\mu_t$ and $\nu$, which takes the following form: $\sum_{i = 1}^{n} \sum_{j = 1}^{m} \pi_{j i} \delta_{(x_j^t, z_i)}$, where $\sum_{i = 1}^{n} \pi_{j i} = \frac{1}{m}$ and $\sum_{j = 1}^{m} \pi_{j i} = a_i$. Note that $\sum_{i = 1}^{n} \sum_{j = 1}^{m} \pi_{j i} \delta_{(x_j^{t + 1}, z_i)}$ is a coupling between $\nu$ and $\mu_{t + 1}$ as well. Hence, 
\begin{align*}
	W_2^2(\mu_{t + 1}, \nu) - W_2^2(\mu_t, \nu) 
	& \le \sum_{i = 1}^{n} \sum_{j = 1}^{m} \|x_j^{t + 1} - z_i\|^2 \pi_{j i} - \sum_{i = 1}^{n} \sum_{j = 1}^{m} \|x_j^t - z_i\|^2 \pi_{j i} \\
	& = \sum_{i = 1}^{n} \sum_{j = 1}^{m} \left(\eta^2 \|\xi_j^t\|^2 - 2 \eta \langle \xi_j^t, x_j^t - z_i \rangle\right) \pi_{j i} \\
	& \le \sum_{i = 1}^{n} \sum_{j = 1}^{m} \left(\eta^2 \|\xi_j^t\|^2 - 2 \eta (V_t(x_j^t) - V_t(z_i))\right) \pi_{j i} \\
	& = \eta^2 \left(\frac{1}{m} \sum_{j = 1}^{m} \|\xi_j^t\|^2\right) - 2 \eta \left(\cV_t(\mu_t) - \cV_t(\nu)\right) \; . 
\end{align*}
Therefore, we have 
\begin{equation}
	\label{eq:convex-upper}
	\cV_t(\mu_t) - \cV_t(\nu) 
	\le 
	\frac{W_2^2(\mu_t, \nu) - W_2^2(\mu_{t + 1}, \nu)}{2 \eta} + \frac{\eta}{2} \cdot \frac{1}{m} \sum_{j = 1}^m \|\xi_j^t\|^2 \; ,
\end{equation}
and \eqref{eq:convex-upper-tele} follows by summing \eqref{eq:convex-upper} iteratively.

\subsection{Proof of Theorem \ref{prop:rand}}\
\label{sec:prop:rand}
We first prove the following:
\begin{equation}
	\label{eq:bound-rand}
	\begin{aligned}
		\E \left[\cV_t(\mu_t) - \cV_t(\nu) ~ | ~ \cF_t \right]
		& \leq \frac{\E \big[W_2^2(\mu_t, \nu) - W_2^2(\mu_{t + 1}, \nu) ~ | ~ \cF_t\big]}{2 \eta} \\
		& \quad + \frac{\eta}{2} \cdot \frac{1}{m} \sum_{j\in S^t} \|\xi_j^t\|^2 + \frac{3}{2} \cdot \frac{1}{m} \sum_{j \notin S^t} \max_{i\in [n]} [V_t(x_j^t) - V_t(z_i)]_+ \; , 
	\end{aligned}
\end{equation}
where $\cF_t$ denotes the $\sigma$-field generated by $\{g_j^s : \forall s < t ~~ \text{and} ~~ \forall j \in S^s\}$ for $t > 1$ and $\cF_1$ is the trivial $\sigma$-field. We write $x_j^{t + 1} = x_j^{t} - v_j^t$, where 
\begin{equation*}
	v^t_j = \sqrt{\eta \frac{\max_{i\in [n]} [V_t(x_j^t) - V_t(z_i)]_+}{d}} g_j^t \quad \text{if} ~~ j \notin S^t
\end{equation*}	
and $v^t_j = \eta \xi_j^t$ otherwise. As in the proof of Theorem \ref{prop:cvx}, let $\sum_{i = 1}^{n} \sum_{j = 1}^{m} \pi_{j i} \delta_{(x_j^t, z_i)}$ be an optimal coupling between $\mu_t$ and $\nu = \sum_{i} a_i \delta_{z_i}$. Again, it is a coupling between $\nu$ and $\mu_{t + 1}$, hence
\begin{align*}
	W_2^2(\mu_{t + 1}, \nu) - W_2^2(\mu_t, \nu) 
	& \le \sum_{i = 1}^{n} \sum_{j = 1}^{m} \|x_j^{t + 1} - z_i\|^2 \pi_{j i} - \sum_{i = 1}^{n} \sum_{j = 1}^{m} \|x_j^t - z_i\|^2 \pi_{j i} \\
	& = \sum_{i = 1}^{n} \sum_{j = 1}^{m} \left(\|v_j^t\|^2 - 2 \langle v_j^t, x_j^t - z_i \rangle\right) \pi_{j i} \\
	& = \frac{1}{m} \sum_{j = 1}^{m} \|v_j^t\|^2 - 2 \sum_{i = 1}^{n} \sum_{j = 1}^{m} \langle v_j^t, x_j^t - z_i \rangle \pi_{j i} \; . 
\end{align*}
Now, taking the conditional expectation $\E[\cdot ~ | ~ \cF_t]$, we have
\begin{equation*}
	\E [W_2^2(\mu_{t + 1}, \nu) - W_2^2(\mu_t, \nu) ~ | ~ \cF_t]
	\le \frac{1}{m} \sum_{j = 1}^{m} \E \left[ \|v_j^t\|^2 ~ | ~ \cF_t\right]- 2 \sum_{i = 1}^{n} \sum_{j = 1}^{m} \E \left[\langle v_j^t, x_j^t - z_i \rangle \pi_{j i} ~ | ~ \cF_t\right] \; . 
\end{equation*}
Now, we upper bound the last two terms.	For $j \in S^t$, we have $\E[ \|v_j^t\|^2 ~ | ~ \cF_t] = \eta^2 \|\xi_j^t\|^2$ because $v_j^t = \xi_j^t$ is measurable with respect to $\cF_t$. Meanwhile, for $j \notin S^t$, since 
$\E \|g_j^t\|^2 = d$, 
\begin{equation*}
	\E \big[\|v_j^t\|^2 ~ | ~ \cF_t\big] = \eta \max_{i\in [n]} [V_t(x_j^t) - V_t(z_i)]_+ \; .
\end{equation*}
Hence, 
\begin{equation}
	\label{eq:temp}
	\frac{1}{m} \sum_{j = 1}^{m} \E \left[\|v_j^t\|^2 ~ | ~ \cF_t\right]
	= \frac{\eta^2}{m} \sum_{j \in S^t} \|\xi_j^t\|^2 + \frac{\eta}{m} \sum_{j \notin S^t} \max_{i\in [n]} [V_t(x_j^t) - V_t(z_i)]_+ \; . 
\end{equation}
Next, for $j \in S^t$, recall that
\begin{equation*}
	\langle v^t_j, x_j^t - z_i \rangle \geq \eta \big( V_t(x_j^t)  - V_t(z_i) \big) \; .
\end{equation*}
For $j \notin S^t$, since $\E g_j^t = 0$, 
\begin{equation*}
	\E \left[\langle v^t_j, x_j^t - z_i \rangle ~ | ~ \cF_t \right]  = 0 \ge \eta \big( V_t(x_j^t) - V_t(z_i) \big) - \eta \max_{i\in [n]} [V_t(x_j^t) - V_t(z_i)]_+ \; .
\end{equation*}
Hence, 
\begin{align*}
	& - 2 \sum_{i = 1}^{n} \sum_{j = 1}^{m} \E \left[\langle v_j^t, x_j^t - z_i \rangle \pi_{j i} ~ | ~ \cF_t\right] \\
	& \le - 2 \eta \sum_{i = 1}^{n} \sum_{j \in S^t} \big( V_t(x_j^t)  - V_t(z_i) \big) \pi_{j i} \\
	& \quad - 2 \eta \sum_{i = 1}^{n} \sum_{j \notin S^t} \big( V_t(x_j^t)  - V_t(z_i) \big) \pi_{j i} + 2 \eta \sum_{i = 1}^{n} \sum_{j \notin S^t} \max_{i\in [n]} [V_t(x_j^t) - V_t(z_i)]_+ \pi_{j i} \\
	& = - 2 \eta \sum_{i = 1}^{n} \sum_{j = 1}^{m} \big( V_t(x_j^t)  - V_t(z_i) \big) \pi_{j i} + \frac{2 \eta}{m} \sum_{j \notin S^t} \max_{i\in [n]} [V_t(x_j^t) - V_t(z_i)]_+ \\
	& = - 2 \eta \left(\cV_t(\mu_t) - \cV_t(\nu)\right) + \frac{2 \eta}{m} \sum_{j \notin S^t} \max_{i\in [n]} [V_t(x_j^t) - V_t(z_i)]_+ \; .
\end{align*}
Combining this result with \eqref{eq:temp} and using $\E \left[\cV_t(\mu_t) - \cV_t(\nu) ~ | ~ \cF_t \right] = \cV_t(\mu_t) - \cV_t(\nu)$, we obtain \eqref{eq:bound-rand}. Now, taking the expectation to the both sides of \eqref{eq:bound-rand} and sum iteratively to obtain \eqref{eq:bound-rand-total}.

\subsection{Proof of Theorem \ref{prop:shrinking}}
\label{sec:prop:shrinking}
First, using uniform boundedness of $V_t$, we have
\begin{equation*}
	\sum_{j \notin S^t} \E\left[\max_{i\in [n]} [V_t(x_j^t) - V_t(z_i)]_+\right] \le 2 B \sum_{j = 1}^{m} 1_{j \notin S^{t}} \; .
\end{equation*}
Combining this with \eqref{eq:bound-rand} and taking the expectation, we have 
\begin{equation*}
	\E \left[\cV_t(\mu_t) - \cV_t(\nu)\right]
	\leq \frac{\E \big[W_2^2(\mu_t, \nu) - W_2^2(\mu_{t + 1}, \nu)\big]}{2 \eta} + \frac{\eta}{2} \E\left[\frac{1}{m} \sum_{j\in S^t} \|\xi_j^t\|^2\right] + \frac{3 B}{m} \sum_{j = 1}^{m} \E[1_{j \notin S^{t}}] \; .
\end{equation*}
Summing this over $t \in [T]$, we have
\begin{equation}
	\label{eq:shrink-temp}
	\E\left[\sum_{t=1}^T \cV_t(\mu_t) - \sum_{t = 1}^T \cV_t(\nu) \right] \leq \frac{W^2_2(\mu_1, \nu)}{2 \eta} + \frac{\eta}{2} \sum_{t = 1}^T \E\left[\frac{1}{m} \sum_{j \in S^t} \|\xi_j^t\|^2\right] + \frac{3 B}{m} \sum_{t = 1}^{T} \sum_{j = 1}^{m} \E[1_{j \notin S^{t}}] \; .
\end{equation}
Now, we upper bound the last term on the right-hand side. Since $j \notin S^t$ implies $j \notin S^{t-1}$, we have
\begin{align*}
	& \E\left[ 1_{j \notin S^t} ~|~ \cF_{t-1}\right] \\
	& =  \E\left[ 1_{j \notin S^{t-1} ~\text{and}~ j \notin S^{t}} ~|~ \cF_{t-1} \right] \\
	& = \E\left[ 1_{ x_j^{t} \notin R^{t} } ~|~  \cF_{t-1} \right] 1_{x_j^{t-1} \in \Omega \backslash R^{t-1}} \\
	& = \Pr_{g\sim N(0, I)}\left( x_j^{t-1} - \sqrt{\eta \tfrac{\max_{i\in [n]} [V_{t - 1}(x) - V_{t-1}(z_i)]_+}{d}} g \notin R^{t} \right)  1_{x_j^{t-1} \in \Omega \backslash R^{t-1}} \\
	& \leq (1-\gamma ) 1_{x_j^{t-1} \in \Omega \backslash R^{t-1}} \\
	& = (1-\gamma ) 1_{j \notin S^{t-1}} \; ,
\end{align*}
where the inequality is due to Assumption \ref{asmp:shrinking}. Therefore, taking the conditional expectation recursively according to the filtration, we have $\E[1_{j \notin S^{t}}] \le (1 - \gamma)^{t - 1} 1_{j \notin S^1}$. Hence,  
\begin{equation*}
	\sum_{j = 1}^{m} \sum_{t = 1}^{T} \E[1_{j \notin S^{t}}] \le (1 + (1 - \gamma) + \cdots + (1 - \gamma)^{T - 1}) \sum_{j = 1}^{m} 1_{j \notin S^1} \le \gamma^{-1} |[m] \backslash S^1| \; .
\end{equation*}
Combining this with \eqref{eq:shrink-temp}, we obtain \eqref{eq:shrinking-bound}.

\section{Supplementary Explanations}
\subsection{Proof of Theorem \ref{thm:continuous-regret}}
\label{sec:continuous-time}
The proof uses differentiability of Wasserstein distance (Theorem 8.13 of \cite{Villani_2003}). If $\boldsymbol{\xi}_t(x)$ is a $C^1$ function of $x$ and $t$ that is globally bounded, for any $\mu \in \cP_2(\R^d)$, one can calculate that
\begin{equation*}
	\frac{\mathrm{d} W_{2}^{2}(\mu, \mu_{t})}{\mathrm{d} t}\Bigg|_{t=s} 
	= 2 \int\left\langle x-\bt_{\mu_s}^\mu(x), \boldsymbol{\xi}_s(x)\right\rangle \dd{\mu_s(x)} 
	\leq 2 [\cE_s(\mu) - \cE_s(\mu_s)] \; . 
\end{equation*}
The inequality is due to the characterization of Fr\'{e}chet subdifferential in Lemma \ref{def:minimal-selection}. Integrating over $s \in [0, T]$, we obtain the regret bound.

\subsection{Assumptions \ref{asmp:non-expansion} and \ref{asmp:shrinking}}
\label{sec:assumptions}
First, we illustrate the notions of feasible and infeasible regions for some loss functions in one dimension. In Figure\ref{fig:non-cvx}, we plot a w-shape loss function $V$. For simplicity, we set two global minimizers $z_1, z_2$ as grid points, and $x_1, x_2$ as player's decision points. One can see that $x_1$ is feasible where the minimal selection subdifferential $\xi_1$ is the slope of the blue line $l_1$, passing through $(x_1, V(x_1))$, $(z_1, V(z_1))$. The decision point $x_2$, lying in the barrier between two global minimizers, is infeasible: for any line passing through $(x_2, V(x_2))$, it is impossible to have smaller values than the loss function at both grid points $z_1, z_2$. The red lines in the figure are two attempts. 
With this intuition, one can verify that the feasible region is $(-\infty, z_1]\cup [z_2, +\infty)$, whereas the infeasible region is $(z_1, z_2)$.

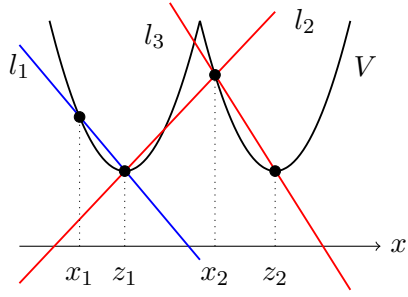
\begin{figure}[ht]
\begin{center}
\begin{tikzpicture}[domain=0:2, scale = 2.0]
\draw[black, line width = 0.25mm]   plot[smooth,domain=0:1] (\x, {4*(\x-0.5)*(\x-0.5)+0.5});
\draw[black, line width = 0.25mm]   plot[smooth,domain=-1:0] (\x, {4*(\x+0.5)*(\x+0.5)+0.5});
\draw[->] (-1.2,0) -- (1.2,0) node[right] {$x$};
\draw[blue, line width = 0.25mm] (-1.2,1.339) -- (0,-0.09);
\node at (-1.2, 1.2) {$l_1$};
\draw[red, line width = 0.25mm] (-1.2, -0.246) -- (0.5, 1.56);
\node at (0.7, 1.5) {$l_2$};
\draw[red, line width = 0.25mm] (-0.2, 1.619) -- (1, -0.29);
\node at (-0.3, 1.4) {$l_3$};
\node at (1.1, 1.2) {$V$};
\filldraw[black] (-0.8,0.86) circle (1pt) node[anchor=west]{};
\draw[dotted] (-0.8,0) -- (-0.8,0.86);
\node at (-0.8, -0.2) {$x_1$};
\filldraw[black] (0.1,1.14) circle (1pt) node[anchor=west]{};
\draw[dotted] (0.1,0) -- (0.1,1.14);
\node at (0.1, -0.2) {$x_2$};
\filldraw[black] (-0.5,0.5) circle (1pt) node[anchor=west]{};
\draw[dotted] (-0.5,0) -- (-0.5,0.5);
\node at (-0.5, -0.2) {$z_1$};
\filldraw[black] (0.5,0.5) circle (1pt) node[anchor=west]{};
\draw[dotted] (0.5,0) -- (0.5,0.5);
\node at (0.5, -0.2) {$z_2$};
\end{tikzpicture}
\end{center}
\caption{An illustration of a w-shape non-convex loss $V\colon \R\to\R$. Two global minimizers $z_1, z_2$ serve as grid points. $x_1, x_2$ are player's decision points. }
\label{fig:non-cvx}
\end{figure}

Next, we showcase a series of w-shape non-convex functions that satisfy Assumptions~\ref{asmp:non-expansion}, \ref{asmp:shrinking}. 

\begin{lemma}
Consider loss functions 
\[
V_t(x) = 
\begin{cases}
a^t (x+1)^2, & \text{ if } x< 0, \\
a^t (x-1)^2, & \text{ if } x\ge 0,
\end{cases}
\]
and grid points $-1 = z_1 < z_2<\dots < z_n = 1$ for some arbitrary $n$. 
If there exists $\epsilon>0$ such that $a^t\ge \epsilon$, $a^t\eta < \frac{1}{2}$ for all $t$, then Assumptions~\ref{asmp:non-expansion}, \ref{asmp:shrinking} hold. 
\end{lemma}

\begin{proof}
	The feasible region is $R^t = (-\infty, -1]\cup [1, \infty)$ for all $t$. 
	We first verify Assumption~\ref{asmp:non-expansion}. 
	Without loss of generality, we consider positive feasible points $x^t = 1 + \delta^t \in R^t$ for some $\delta^t\ge0$. 
	The minimal selection of subdifferential returns $\xi_x = 2a^t\delta^t$. Therefore $x^{t+1} = x^t - 2\eta a^t \delta^t = 1 + (1 - 2a^t\eta)\delta$. Under the condition $a^t\eta<\frac{1}{2}$, we have $x^{t+1}\ge1$, i.e., $R^t_\eta\subset R^{t+1}$.
	
	Now, we verify Assumption~\ref{asmp:shrinking}.
	For any infeasible $x^t\in[0, 1)$, update by random exploration
	\[
	x^{t+1} = x_t - \sqrt{\eta V_t(x^t)}g \; ,
	\]
	where $g\sim\mathcal{N}(0, 1)$. 
	Then, 
	\begin{align*}
		\mathbb{P}(x^{t+1}\in R^{t+1}) &\ge \mathbb{P}(x^{t} - \sqrt{\eta V_t(x^t)}g\ge 1) \\
		&= \mathbb{P}( - \sqrt{\eta a^t}(1-x^t)g\ge 1 - x^t) \\
		&= \mathbb{P}\left(g\le -\frac{1}{\sqrt{a^t \eta}}\right) \\
		&\ge \mathbb{P}\left(g\le -\frac{1}{\sqrt{\epsilon \eta}}\right)\; ,
	\end{align*}
	where the last inequality is due to the condition $a^t\ge \epsilon$. 
	Note that the lower bound above holds for any $x\in[0, 1)$. By a symmetric argument, we conclude that Assumption~\ref{asmp:shrinking} is satisfied with $\gamma = \mathbb{P}(g\le -1/\sqrt{\epsilon \eta})$. 
\end{proof}

\subsection{Results in Subsection \ref{sec:extensions}}
\label{sec:appendix-extensions}
Here, we provide detailed explanations on the results presented in Subsection \ref{sec:extensions}. First, using a similar EVI technique, we derive the following regret bound for the relaxed minimal selection algorithm.

\begin{theorem}
	\label{prop:slack}
	Assume $\Omega = \R^d$. Let $\mu_t$, $\xi_j^t$, and $s_j^t$ be the output of the relaxed minimal selection algorithm. For any $V_t$ and any measure $\nu \in \cP_2(\R^d)$ supported on $Z$,
	\begin{equation}
		\label{eq:bound-slack}
		\sum_{t = 1}^{T} \cV_t(\mu_t) - \sum_{t = 1}^{T} \cV_t(\nu) 
		\le 
		\frac{W_2^2(\mu_1, \nu) - W_2^2(\mu_{T + 1}, \nu)}{2 \eta}
		+ \frac{\eta}{2} \sum_{t = 1}^{T} \frac{1}{m} \sum_{j = 1}^{m} \left(\|\xi_j^t\|^2 + \frac{2 s_j^t}{\eta}\right) \; .
	\end{equation}
\end{theorem}

\begin{proof}
	As in the proof of Theorem \ref{prop:cvx}, let $\sum_{i = 1}^{n} \sum_{j = 1}^{m} \pi_{j i} \delta_{(x_j^t, z_i)}$ be an optimal coupling between $\mu_t$ and $\nu = \sum_{i} a_i \delta_{z_i}$. Again, it is a coupling between $\nu$ and $\mu_{t + 1}$, hence
	\begin{align*}
		W_2^2(\mu_{t + 1}, \nu) - W_2^2(\mu_t, \nu) 
		& \le \sum_{i = 1}^{n} \sum_{j = 1}^{m} \|x_j^{t + 1} - z_i\|^2 \pi_{j i} - \sum_{i = 1}^{n} \sum_{j = 1}^{m} \|x_j^t - z_i\|^2 \pi_{j i} \\
		& = \sum_{i = 1}^{n} \sum_{j = 1}^{m} \left(\eta^2 \|\xi_j^t\|^2 - 2 \eta \langle \xi_j^t, x_j^t - z_i \rangle\right) \pi_{j i} \\
		& \le \sum_{i = 1}^{n} \sum_{j = 1}^{m} \left(\eta^2 \|\xi_j^t\|^2 + 2 \eta s_j^t - 2 \eta (V_t(x_j^t) - V_t(z_i))\right) \pi_{j i} \\
		& = \frac{\eta^2}{m} \sum_{j = 1}^{m} \left(\|\xi_j^t\|^2 + \frac{2 s_j^t}{\eta}\right) - 2 \eta \left(\cV_t(\mu_t) - \cV_t(\nu)\right) \; .
	\end{align*}
	Hence, 
	\begin{equation}
		\label{eq:bound-slack-1}
		\cV_t(\mu_t) - \cV_t(\nu) 
		\le \frac{W_2^2(\mu_t, \nu) - W_2^2(\mu_{t + 1}, \nu)}{2 \eta} + \frac{\eta}{2} \cdot \frac{1}{m} \sum_{j = 1}^{m} \left(\|\xi_j^t\|^2 + \frac{2 s_j^t}{\eta}\right)
	\end{equation}
	\eqref{eq:bound-slack} follows by summing \eqref{eq:bound-slack-1} iteratively.
\end{proof}

In fact, we can further upper bound \eqref{eq:bound-slack-1}. Note that $\xi_j = 0$ and $s_j = \max_{i \in [n]} [V_t(x_j^t) - V_t(z_i)]_+$ satisfy the constraint \eqref{eq:constraint-slack}, hence the minimizer $(\xi_j^t, s_j^t)$ satisfies
\begin{equation*}
	\|\xi_j^t\|^2 + \frac{2 s_j^t}{\eta} \le \frac{2 \max_{i \in [n]} [V_t(x_j^t) - V_t(z_i)]_+}{\eta} \; .
\end{equation*}
Using this upper bound for $j \notin S^t$, we have 
\begin{equation}
	\label{eq:slack-bound-rand}
	\begin{aligned}
		\cV_t(\mu_t) - \cV_t(\nu) 
		& \le \frac{W_2^2(\mu_t, \nu) - W_2^2(\mu_{t + 1}, \nu)}{2 \eta} \\ 
		& \quad + \frac{\eta}{2} \cdot \frac{1}{m} \sum_{j \in S^t} \|\xi_j^t\|^2 + \frac{1}{m} \sum_{j \notin S^t} \max_{i \in [n]} [V_t(x_j^t) - V_t(z_i)]_+ \; .		
	\end{aligned}
\end{equation}
Notice that this upper bound is comparable to \eqref{eq:bound-rand}. Therefore, the relaxed minimal selection algorithm and the minimal selection or exploration algorithm admit essentially the same upper bound.

Next, we present a regret bound for the case where the loss is given according to the interaction functional. 

\begin{theorem}
	\label{prop:interaction}
	Assume $\Omega = \R^d$ and consider the game with the loss \eqref{eq:interaction-loss} discussed in Subsection \ref{sec:extensions}. Let $\mu_t$ and $\xi_j^t$ be the output of the minimal selection algorithm with the constraint \eqref{eq:constraint-interaction}. If $W^t$ is convex, for any measure $\nu \in \cP_2(\R^d)$ supported on $Z$,
	\begin{equation}
		\label{eq:interaction-bound}
		\sum_{t = 1}^{T} \cW_t(\mu_t) - \sum_{t = 1}^{T} \cW_t(\nu) 
		\le 
		\frac{W_2^2(\mu_1, \nu) - W_2^2(\mu_{T + 1}, \nu)}{2 \eta} + \frac{\eta}{2} \sum_{t = 1}^{T} \left(\frac{1}{m} \sum_{j = 1}^m \|\xi_j^t\|^2\right) \; .
	\end{equation}
\end{theorem}
\begin{proof}
	As in the proof of Theorem \ref{prop:cvx}, let $\sum_{i = 1}^{n} \sum_{j = 1}^{m} \pi_{j i} \delta_{(x_j^t, z_i)}$ be an optimal coupling between $\mu_t$ and $\nu = \sum_{i} a_i \delta_{z_i}$. Again, it is a coupling between $\nu$ and $\mu_{t + 1}$, hence
	\begin{align*}
		W_2^2(\mu_{t + 1}, \nu) - W_2^2(\mu_t, \nu) 
		& \le \sum_{i = 1}^{n} \sum_{j = 1}^{m} \|x_j^{t + 1} - z_i\|^2 \pi_{j i} - \sum_{i = 1}^{n} \sum_{j = 1}^{m} \|x_j^t - z_i\|^2 \pi_{j i} \\
		& = \sum_{i = 1}^{n} \sum_{j = 1}^{m} \left(\eta^2 \|\xi_j^t\|^2 - 2 \eta \langle \xi_j^t, x_j^t - z_i \rangle\right) \pi_{j i} \; .
	\end{align*}
	Using \eqref{eq:constraint-interaction},
	\begin{align*}
		\sum_{i = 1}^{n} \sum_{j = 1}^{m} - 2 \eta \langle \xi_j^t, x_j^t - z_i \rangle \pi_{j i} 
		& \le \sum_{i = 1}^{n} \sum_{j = 1}^{m} - 2 \eta \left(\frac{1}{m} \sum_{k = 1}^{m} W^t(x_j^t - x_k^t) - \min_{k \in [n]} W^t(z_i - z_k)\right) \pi_{j i} \\
		& = - 2 \eta \left(\frac{1}{m^2} \sum_{j, k = 1}^{m} W^t(x_j^t - x_k^t) - \sum_{i = 1}^{n} a_i \cdot \min_{k \in [n]} W^t(z_i - z_k)\right) \\
		& \le - 2 \eta \left(\frac{1}{m^2} \sum_{j, k = 1}^{m} W^t(x_j^t - x_k^t) - \sum_{i = 1}^{n} a_i \sum_{k = 1}^{n} a_k W^t(z_i - z_k)\right) \\
		& = - 2 \eta \left(\cW_t(\mu_t) - \cW_t(\nu)\right) \; . 
	\end{align*}
	Therefore, we have 
	\begin{equation}
		\label{eq:interaction-1}
		\cW_t(\mu_t) - \cW_t(\nu) 
		\le 
		\frac{W_2^2(\mu_t, \nu) - W_2^2(\mu_{t + 1}, \nu)}{2 \eta} + \frac{\eta}{2} \cdot \frac{1}{m} \sum_{j = 1}^m \|\xi_j^t\|^2 \; ,
	\end{equation}
	and \eqref{eq:interaction-bound} follows by summing \eqref{eq:interaction-1} iteratively.
\end{proof}

Lastly, we prove that the projection step does not affect the regret bound.
\begin{theorem}
	\label{prop:projection}
	Assume $\Omega$ is closed and convex. Let $\mu_t$ and $\xi_j^t$ be the output of the minimal selection algorithm with a modified update rule $x_j^{t + 1} = P_{\Omega}(x_j^t - \eta \xi_j^t)$. If $V_t$ is convex, for any measure $\nu \in \cP_2(\R^d)$ supported on $Z$,
	\begin{equation*}
		\sum_{t = 1}^{T} \cV_t(\mu_t) - \sum_{t = 1}^{T} \cV_t(\nu) 
		\le 
		\frac{W_2^2(\mu_1, \nu) - W_2^2(\mu_{T + 1}, \nu)}{2 \eta} + \frac{\eta}{2} \sum_{t = 1}^{T} \left(\frac{1}{m} \sum_{j = 1}^m \|\xi_j^t\|^2\right) \; .
	\end{equation*}
\end{theorem}

\begin{proof}
	As in the proof of Theorem \ref{prop:cvx}, let $\sum_{i = 1}^{n} \sum_{j = 1}^{m} \pi_{j i} \delta_{(x_j^t, z_i)}$ be an optimal coupling between $\mu_t$ and $\nu = \sum_{i} a_i \delta_{z_i}$. Again, it is a coupling between $\nu$ and $\mu_{t + 1}$, hence
	\begin{equation*}
		W_2^2(\mu_{t + 1}, \nu) - W_2^2(\mu_t, \nu) 
		\le \sum_{i = 1}^{n} \sum_{j = 1}^{m} \|x_j^{t + 1} - z_i\|^2 \pi_{j i} - \sum_{i = 1}^{n} \sum_{j = 1}^{m} \|x_j^t - z_i\|^2 \pi_{j i} \; .
	\end{equation*}
	One property of the projection $P_{\Omega}$ is that $\|x_j^{t + 1} - \omega\| \le \|x_j^t - \eta \xi_j^t - \omega\|$ for all $\omega \in \Omega$. Thus, 	
	\begin{align*}		
		W_2^2(\mu_{t + 1}, \nu) - W_2^2(\mu_t, \nu) 
		& \le \sum_{i = 1}^{n} \sum_{j = 1}^{m} \|x_j^{t + 1} - z_i\|^2 \pi_{j i} - \sum_{i = 1}^{n} \sum_{j = 1}^{m} \|x_j^t - z_i\|^2 \pi_{j i} \\
		& \le \sum_{i = 1}^{n} \sum_{j = 1}^{m} \|x_j^t - \eta \xi_j^t - z_i\|^2 \pi_{j i} - \sum_{i = 1}^{n} \sum_{j = 1}^{m} \|x_j^t - z_i\|^2 \pi_{j i} \\
		& = \sum_{i = 1}^{n} \sum_{j = 1}^{m} \left(\eta^2 \|\xi_j^t\|^2 - 2 \eta \langle \xi_j^t, x_j^t - z_i \rangle\right) \pi_{j i} \\
		& \le \eta^2 \left(\frac{1}{m} \sum_{j = 1}^{m} \|\xi_j^t\|^2\right) - 2 \eta \left(\cV_t(\mu_t) - \cV_t(\nu)\right) \; ,
	\end{align*}
	where the last inequality directly follows as in the proof of Theorem \ref{prop:cvx}. Therefore, the regret bounds in Theorem \ref{prop:cvx} still hold. 
\end{proof}

\begin{remark}
	We can apply this projection step to the MSoE algorithm and the relaxed minimal selection algorithm as well; we modify their update rules by combining them with $P_{\Omega}$. Then, the regret bounds in Theorems \ref{prop:rand} and \ref{prop:slack} still hold. 
\end{remark}

\section{Simulations}\label{sec:simulations}
In this section, we examine the empirical performance of the minimal selection and the MSoE algorithms. Here, we consider a toy example where $\Omega = \{x \in \R^2 : \|x\| \le 1\}$, $m = 10$, and $Z$ consists of uniform grid points of $\Omega$ with $n = 797$, obtained by forming uniform grid points of $[-1, 1]^2$ and choose a subset of included in $\Omega$. Code for reproducing all the results below is provided in the supplementary material.

\paragraph{Convex case} First, we consider a simple quadratic function $V_t(x) = \|x - u_t\|^2$, where $u_1 = (-\frac{1}{\sqrt{2}}, -\frac{1}{\sqrt{2}})$ and $u_t = u_1 + (0.15 t, 0.15 t)$. For better understanding, we may imagine a situation, where we deploy $m$ drones to track a moving target $u_t$ using the signals $V_t$ (distances to the target) captured at $Z$ (fixed stations with sensors) and $\{x_j^t\}_{j \in [m]}$ (drones with mobile sensors). Figure \ref{fig:convex_ms} shows how the minimal selection algorithm moves the decision points. At $t = 1$ (Figure \ref{fig:convex_ms}(a)), the decision points, which are randomly initialized, start moving towards the darker region based on $\xi_j^t$'s from the minimal selection algorithm. In this simple toy example, we can see that the decision points quickly gather around the minimum of $V_t$ (the target) and follow it.

\begin{figure}[ht]
    \centering
	\subfloat[$t = 1$]{\includegraphics[width=0.24\textwidth]{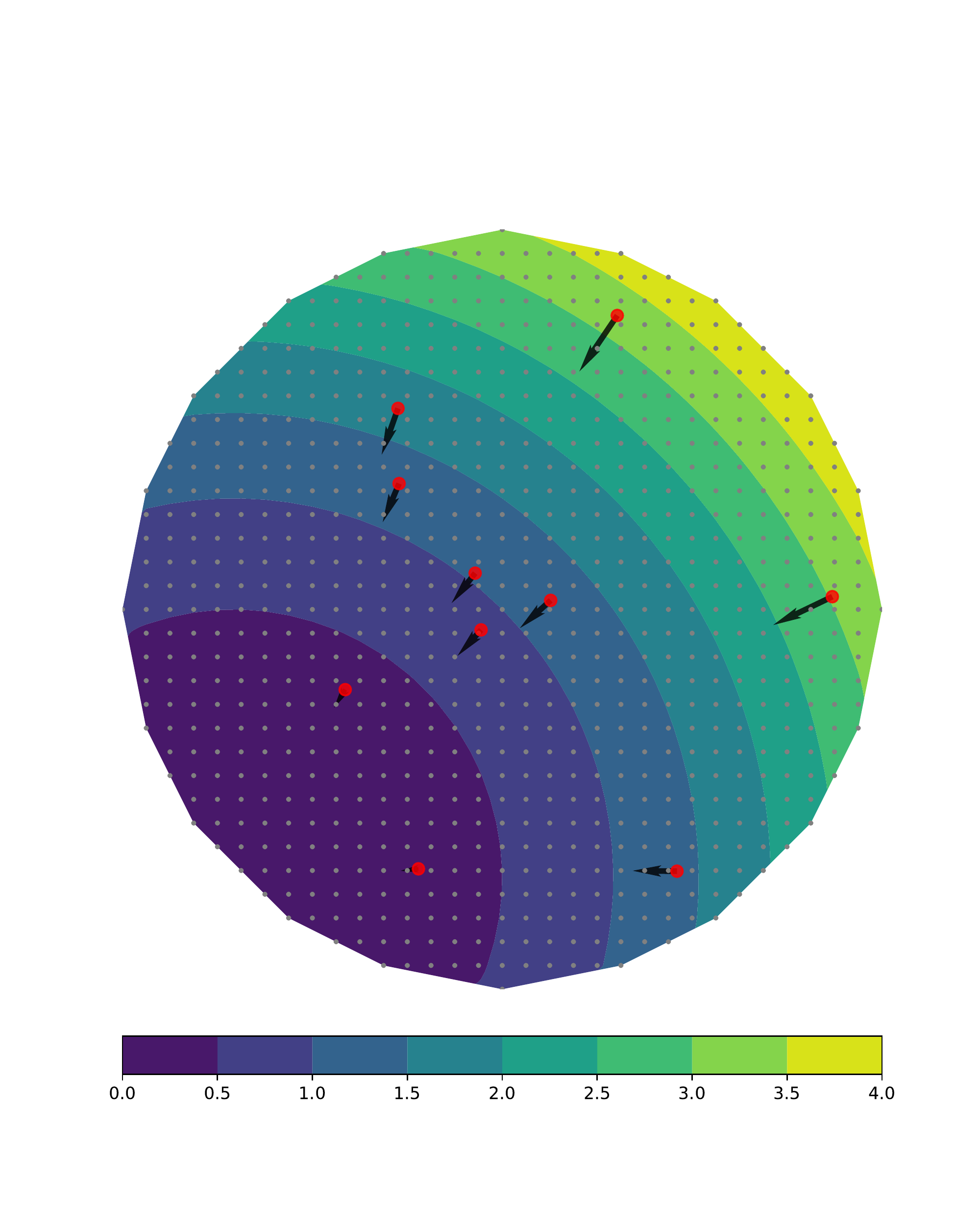}} 
	\subfloat[$t = 3$]{\includegraphics[width=0.24\textwidth]{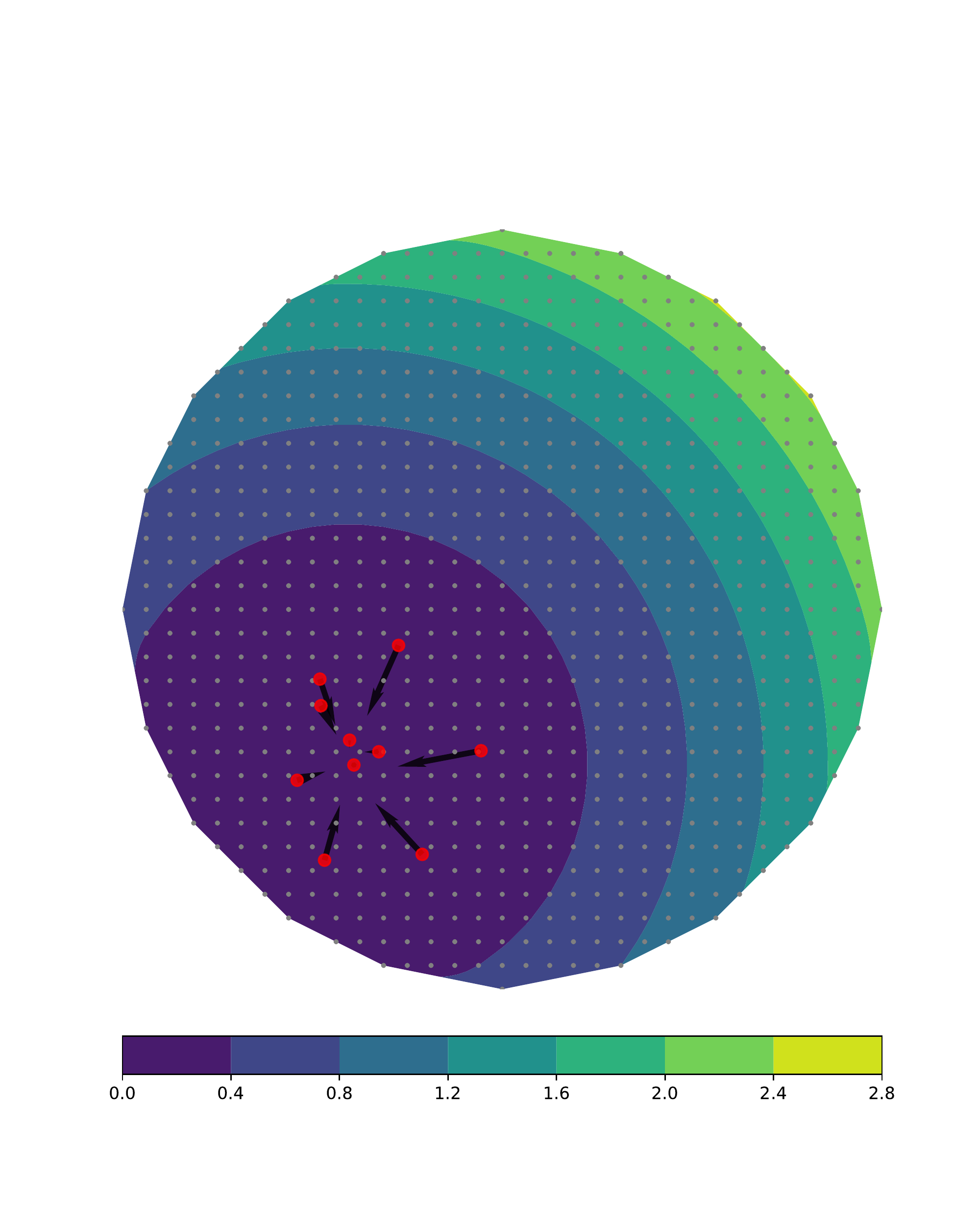}}
	\subfloat[$t = 5$]{\includegraphics[width=0.24\textwidth]{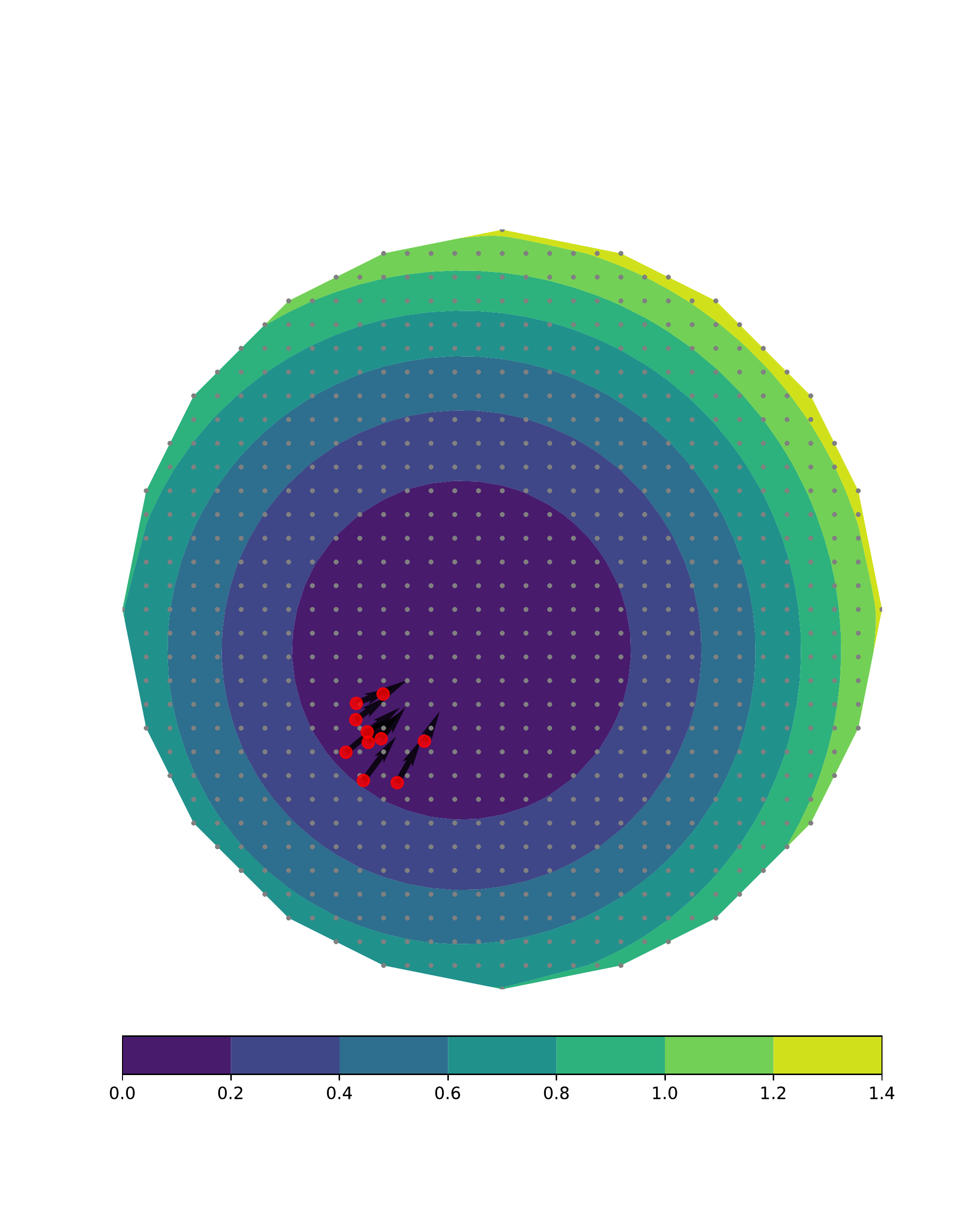}}
	\subfloat[$t = 7$]{\includegraphics[width=0.24\textwidth]{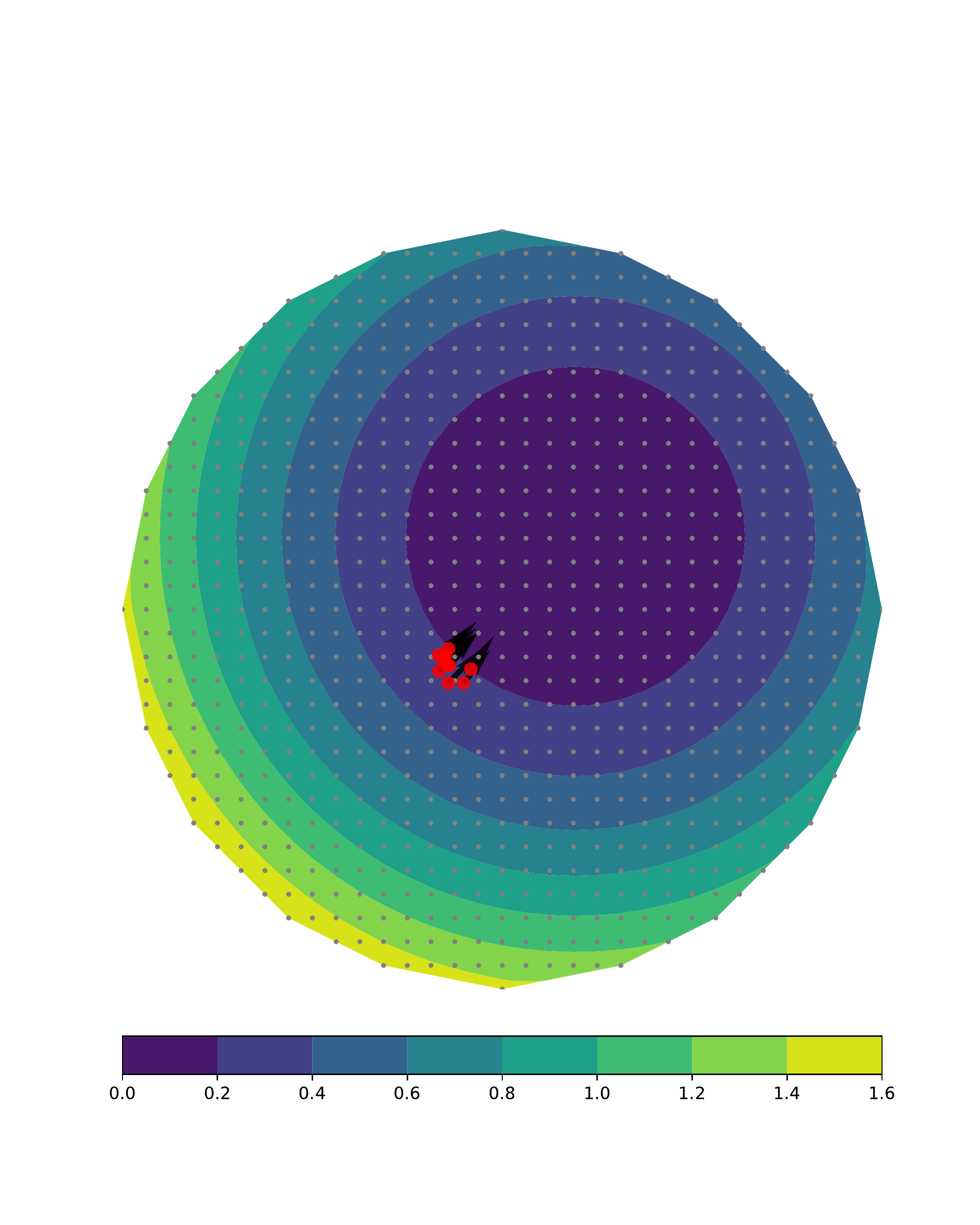}}
	\caption{Minimal selection algorithm for convex $V_t$ with stepsize $\eta = 0.2$. The gray dots and the red circles denote $Z$ and $\{x_j^t\}_{j \in [m]}$, respectively. The black solid arrows show $\xi_j^t$'s. The contour regions represent the level of $V_t$ (darker = smaller as shown in the horizontal colorbars).}
	\label{fig:convex_ms}
\end{figure}

\paragraph{Non-convex case} As a simple non-convex example, consider $V_t(x) = \min\{\|x - u_t\|^2,  \|x - v_t\|^2\}$, where $u_1 = v_1 = (-\frac{1}{\sqrt{2}}, -\frac{1}{\sqrt{2}})$, $u_t = u_1 + (0.165 t, 0.11 t)$, and $v_t = v_1 + (0.11 t, 0.165 t)$. As in the convex case, we may interpret $u_t$ and $v_t$ as moving targets, while the signal is the distance to a closer target. Figure \ref{fig:nonconvex_msoe} shows how the MSoE algorithm works. As opposed to the convex case, we can check that there are infeasible decision points ($j \notin S^t$ as in Definition \ref{def:MSoE}) after $t = 3$. The MSoE algorithm let such points move along random directions, thereby continuing the tracking situation; the plain minimal selection would have ended up stopping all the decision points, losing the targets. Although infeasible points might get far away from the targets as in Figure \ref{fig:nonconvex_msoe}(e) and Figure \ref{fig:nonconvex_msoe}(f), once they get into a feasible region, they start moving again towards the darker area based on $\xi_j^t$'s from the minimal selection. Figure \ref{fig:nonconvex_msoe}(g) and Figure \ref{fig:nonconvex_msoe}(h) show that all the decision points somehow get closer to the darker area after repeating the minimal selection and exploration.

\begin{figure}[ht]
    \centering
	\subfloat[$t = 1$]{\includegraphics[width=0.24\textwidth]{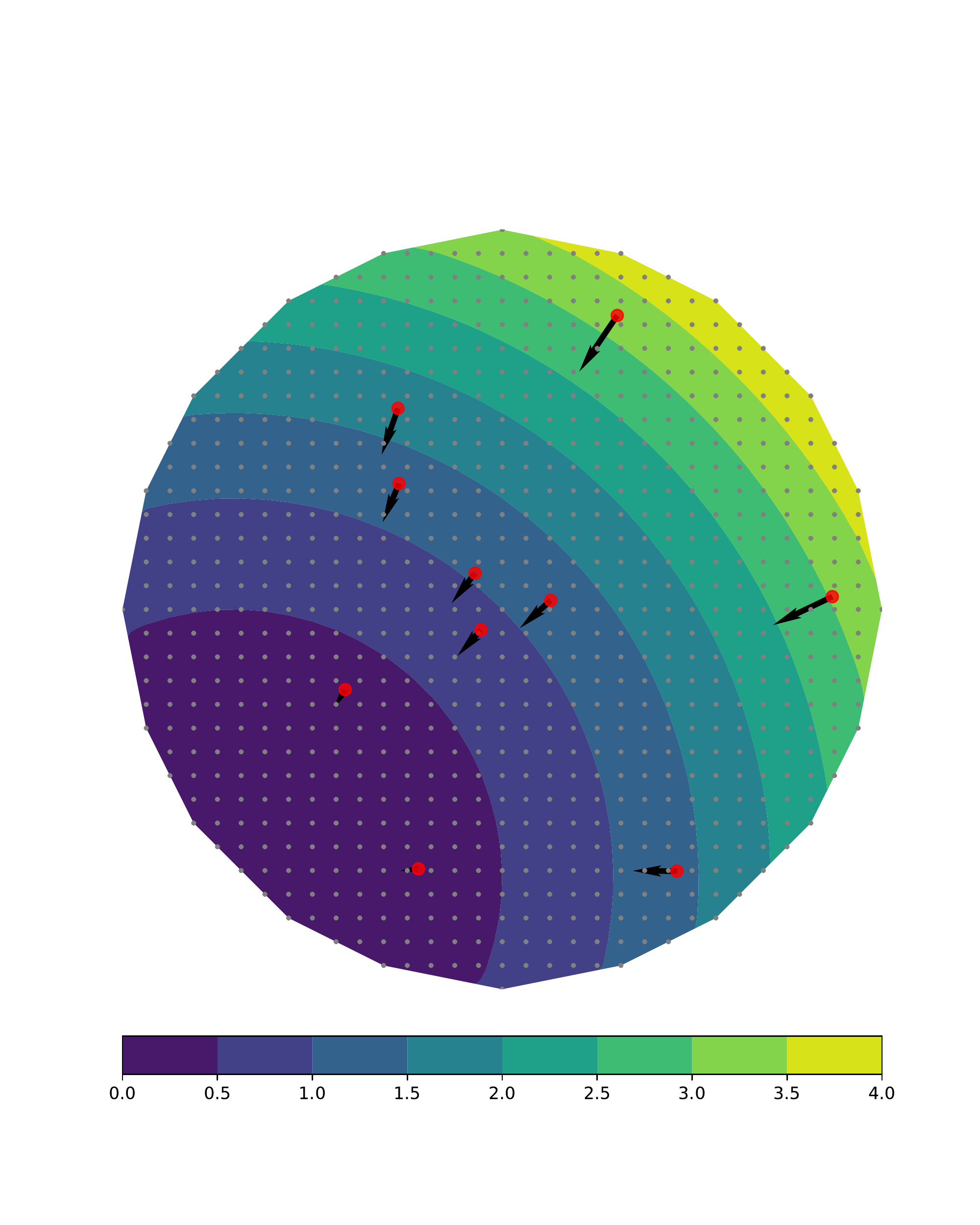}} 
	\subfloat[$t = 3$]{\includegraphics[width=0.24\textwidth]{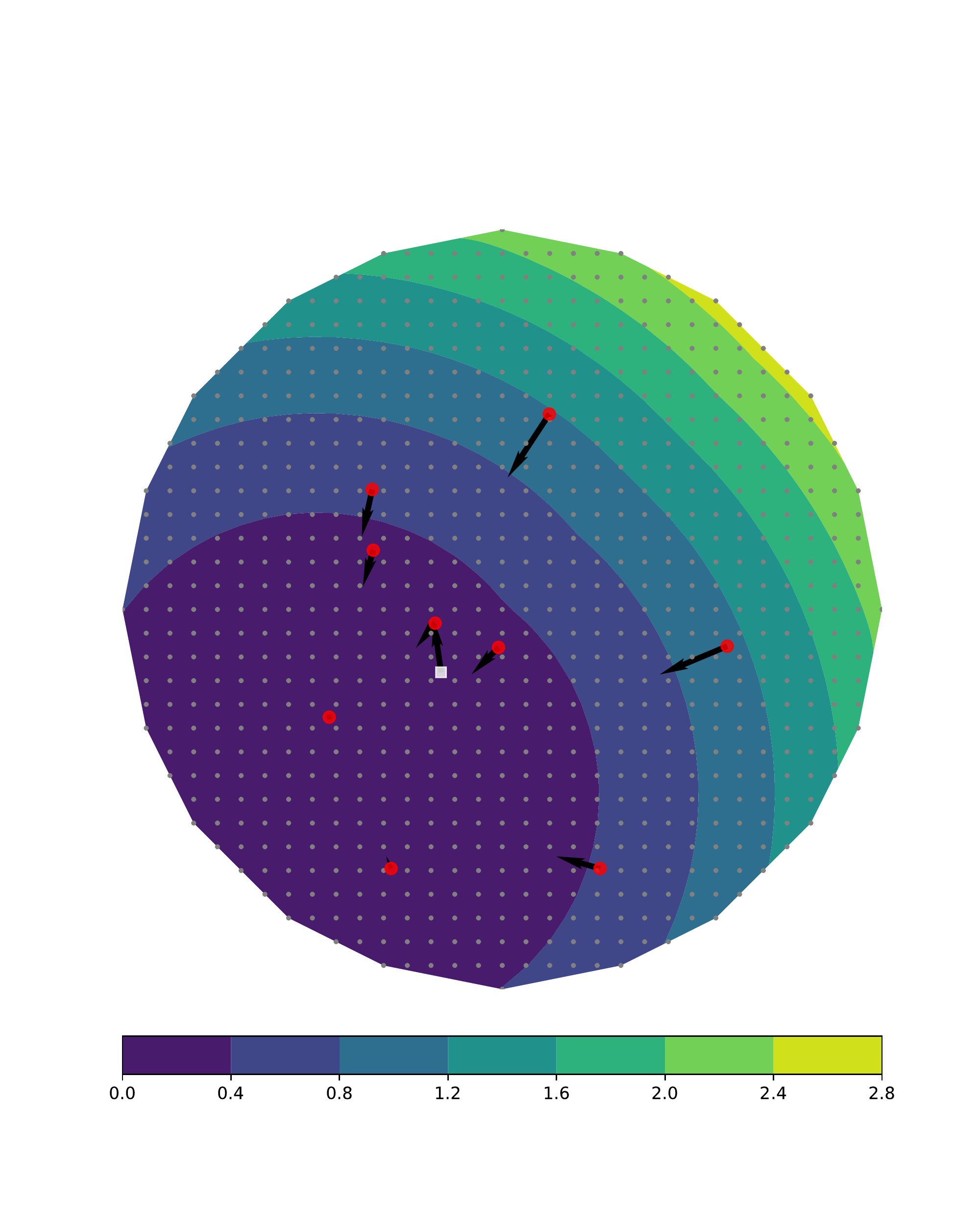}}
	\subfloat[$t = 5$]{\includegraphics[width=0.24\textwidth]{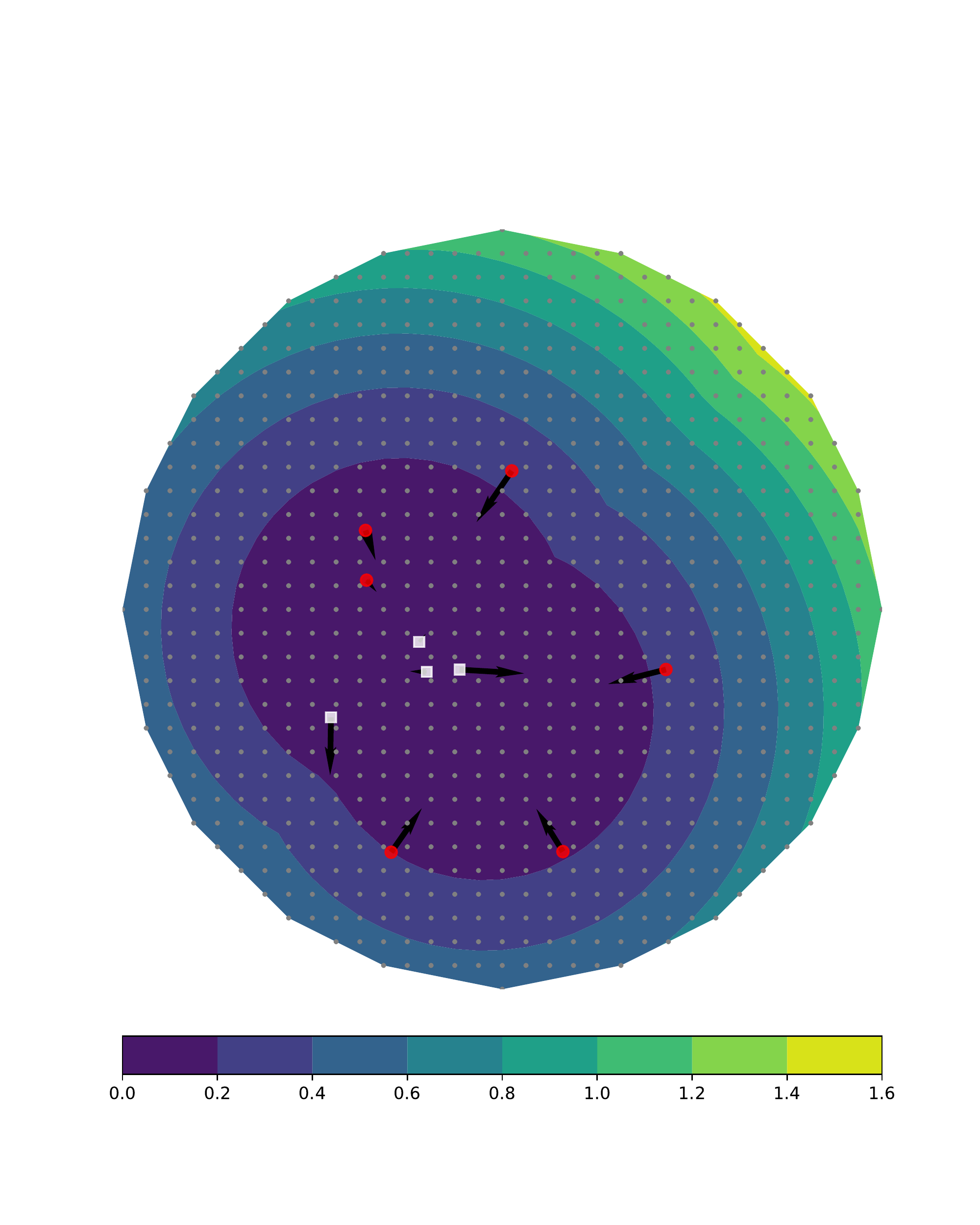}}
	\subfloat[$t = 7$]{\includegraphics[width=0.24\textwidth]{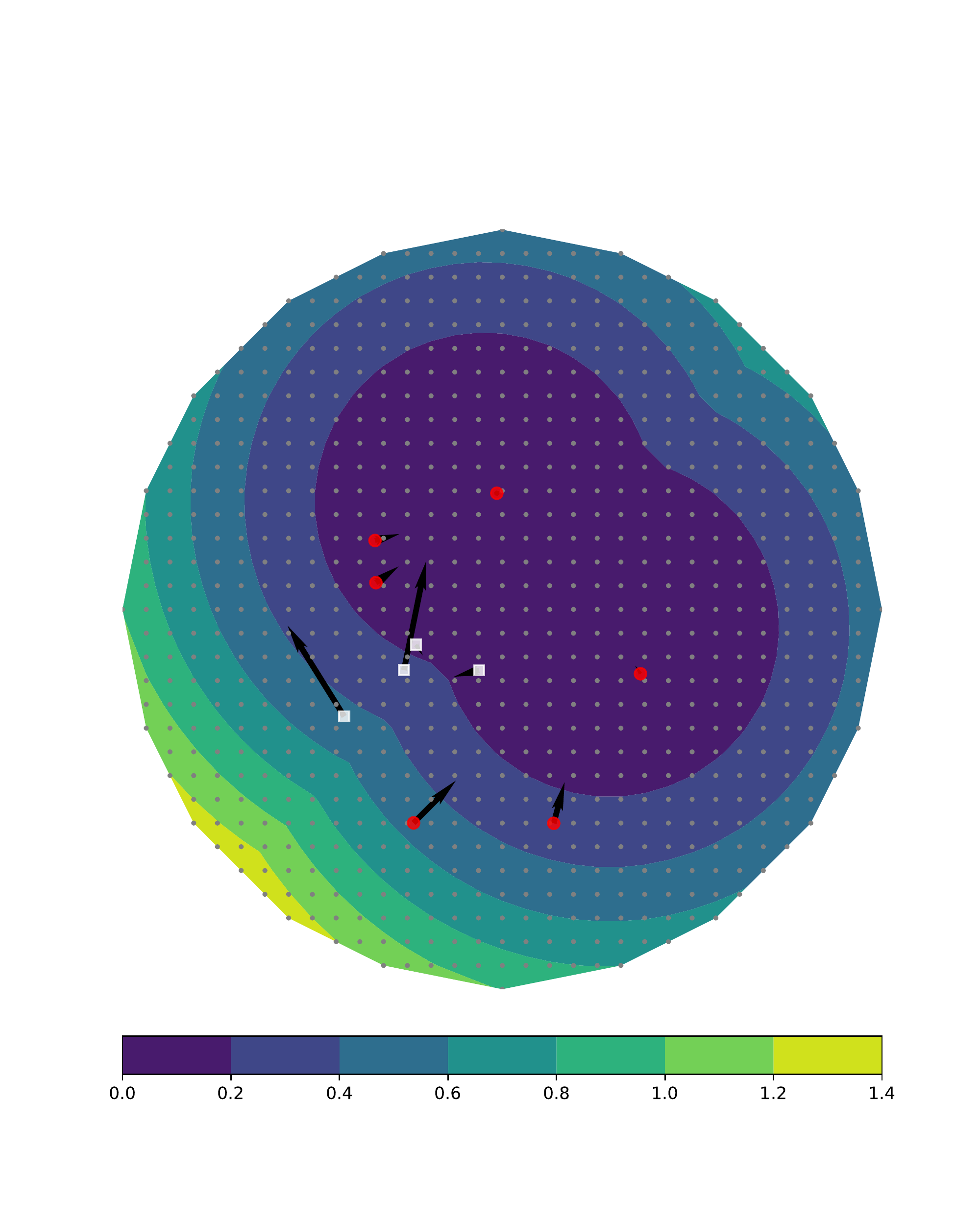}}	\hfill
	\subfloat[$t = 9$]{\includegraphics[width=0.24\textwidth]{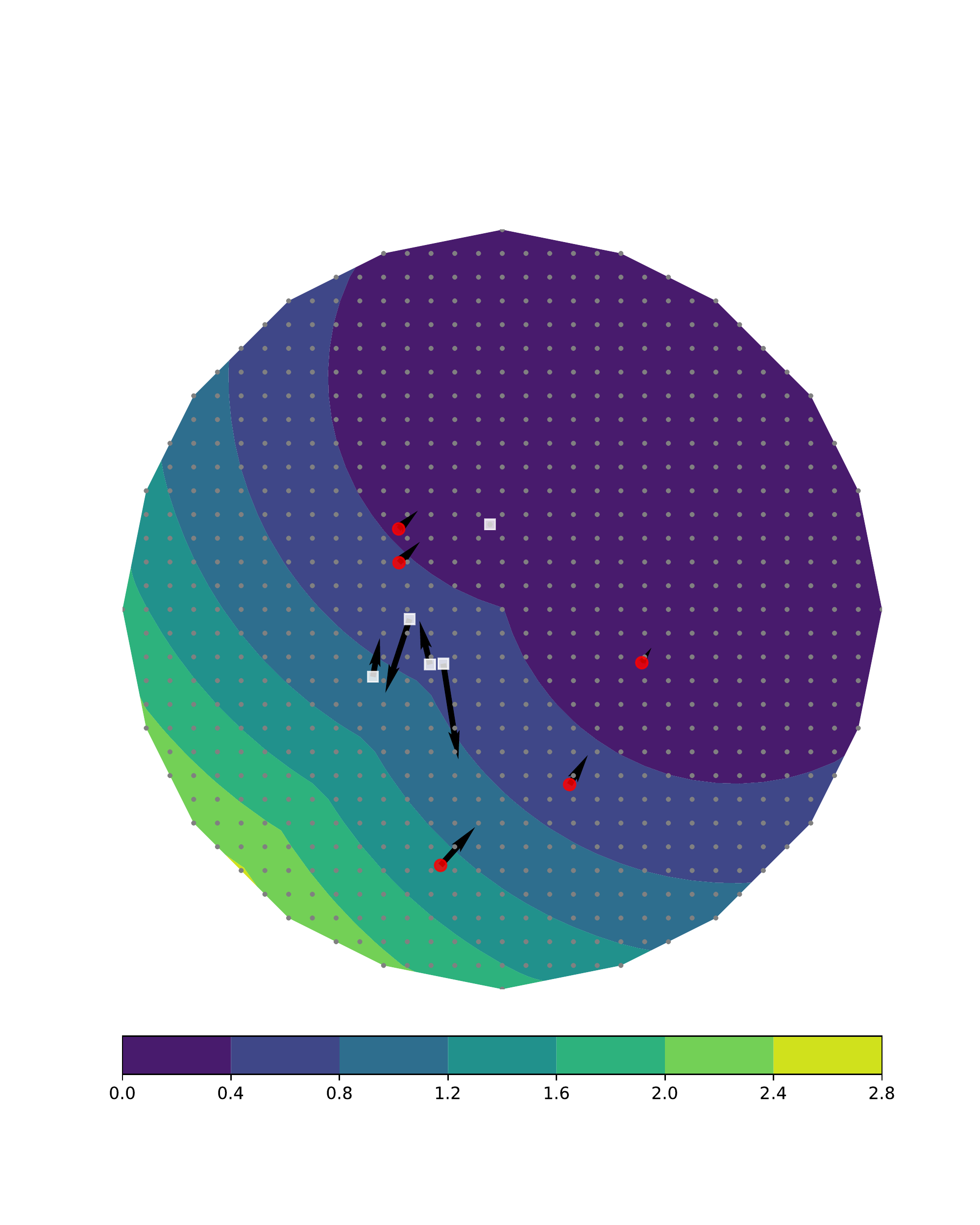}}	
	\subfloat[$t = 12$]{\includegraphics[width=0.24\textwidth]{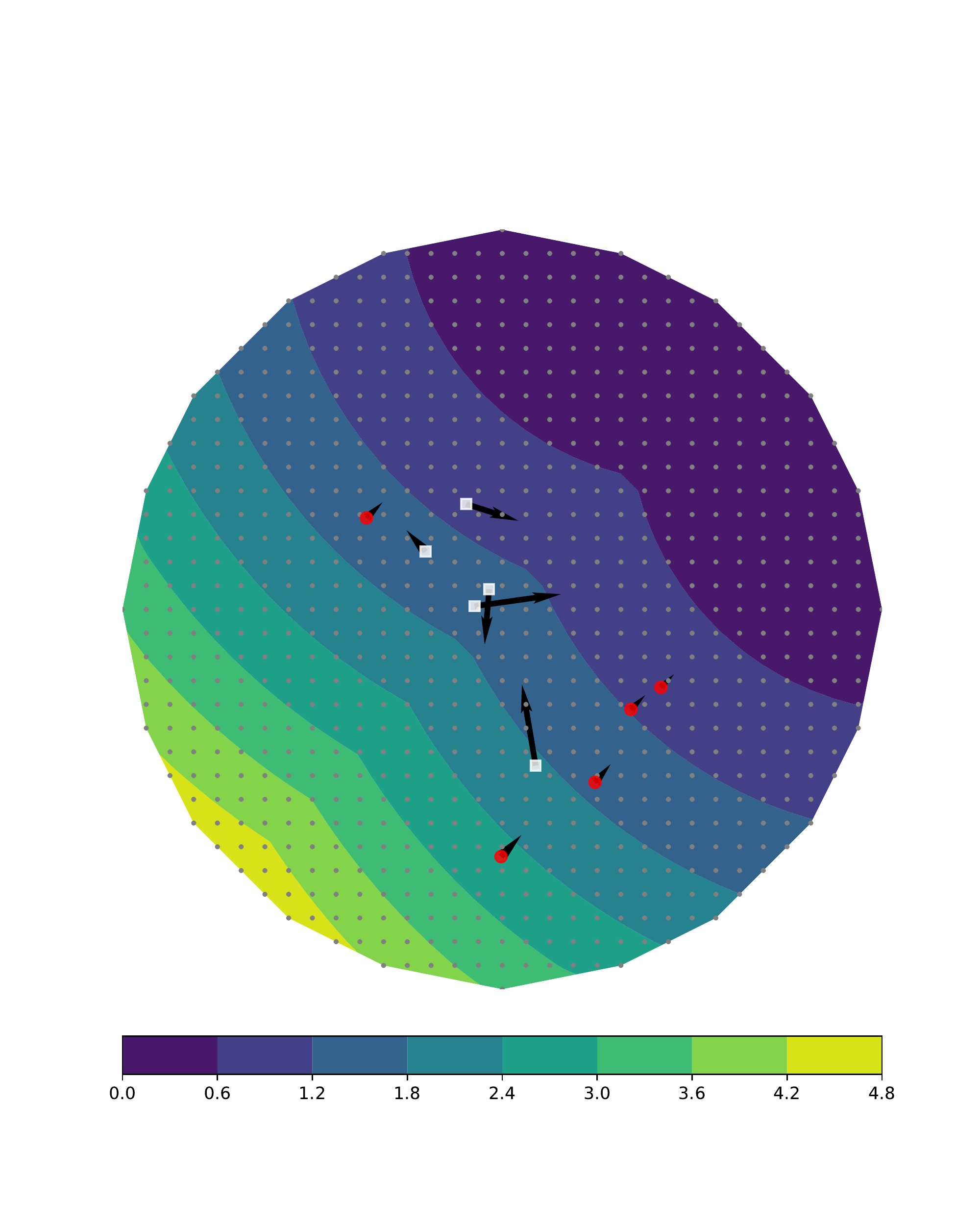}}	
	\subfloat[$t = 18$]{\includegraphics[width=0.24\textwidth]{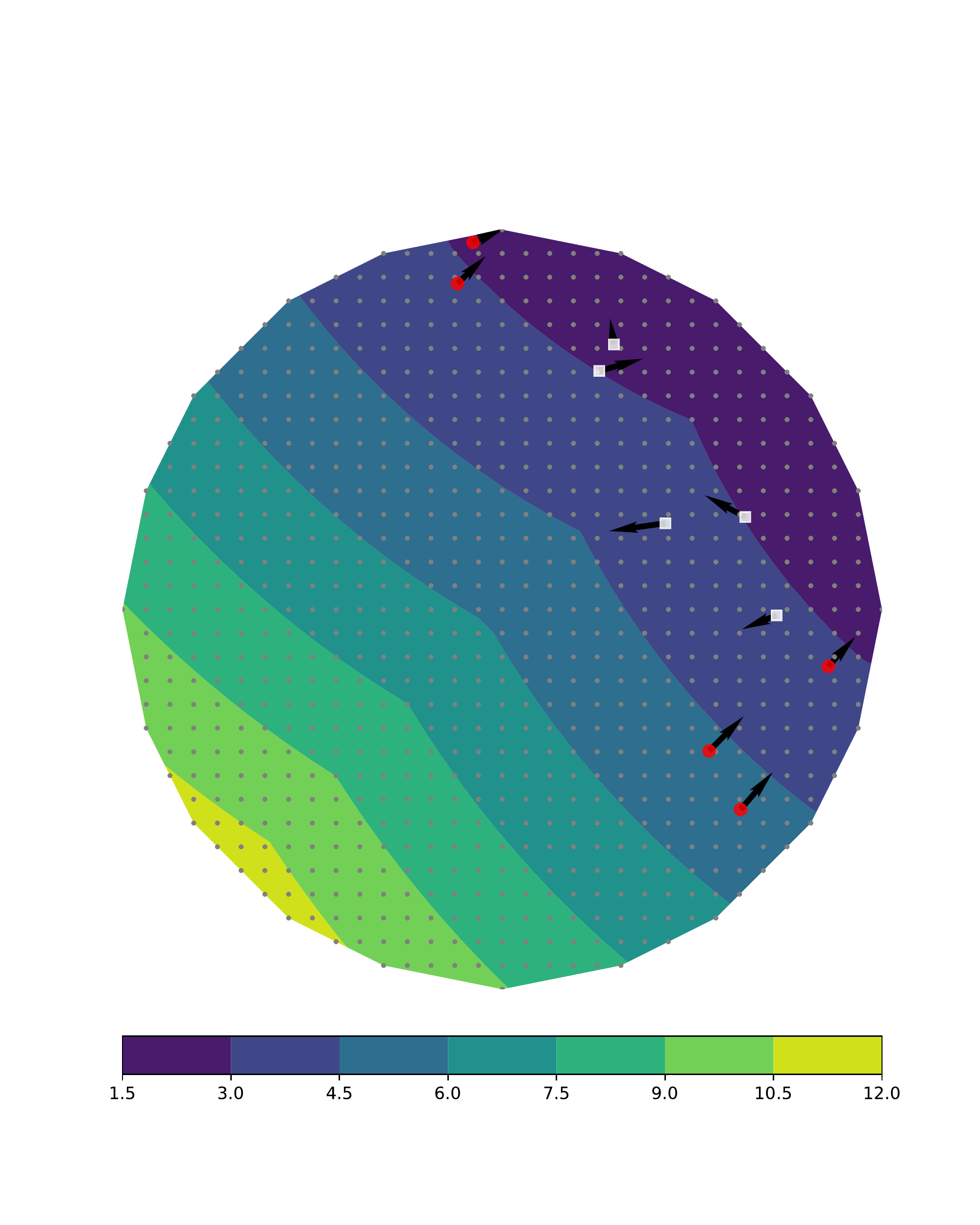}}	
	\subfloat[$t = 19$]{\includegraphics[width=0.24\textwidth]{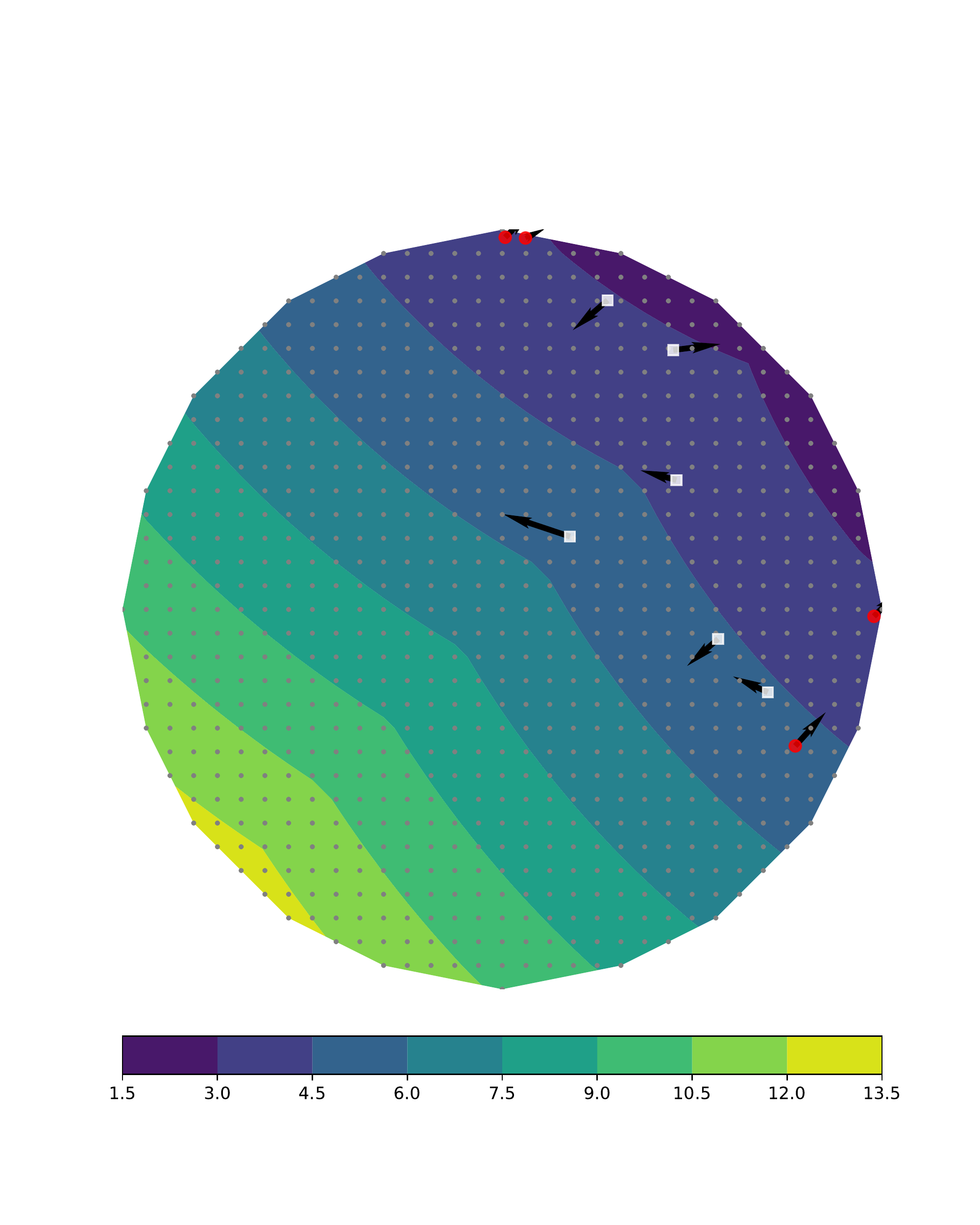}}
	\caption{MSoE algorithm for non-convex case with stepsize $\eta = 0.05$. The red circles denote the feasible decision points ($j \in S^t$) and the white squares denote the infeasible decision points ($j \notin S^t$).}
	\label{fig:nonconvex_msoe}
\end{figure}

\end{document}